\documentclass[lettersize,journal]{IEEEtran}
\usepackage{amsmath,amsfonts}
\usepackage{array}
\usepackage[caption=false,font=normalsize,labelfont=sf,textfont=sf]{subfig}
\usepackage{textcomp}
\usepackage{stfloats}
\usepackage{url}
\usepackage{verbatim}
\usepackage{graphicx}
\usepackage{cite}
\usepackage{placeins}
\usepackage{xcolor}
\usepackage[ruled,linesnumbered,vlined]{algorithm2e}
\usepackage[short,c2,nocomma]{optidef}
\usepackage{etoolbox}
\usepackage{amsthm}
\usepackage{enumitem}

\usepackage{wrapfig}

\newtheorem{assumption}{Assumption}
\newtheorem{theorem}{Theorem}

\newtheorem{lemma}{Lemma}

\makeatletter
\patchcmd{\@algocf@start}
  {-1.5em}
  {0pt}
  {}{}
\makeatother

\newcommand{\ntar}[0]{{n_\textnormal{tar}}}
\newcommand{\agenttrj}[0]{{\tau_\textnormal{a}}}
\newcommand{\agenttrjdot}[0]{{\dot{\tau}_\textnormal{a}}}
\newcommand{\Qagt}[0]{{\mathcal{Q}_\textnormal{a}}}

\newcommand{\qtari}[0]{{q_{\textnormal{tar},i}}}
\newcommand{\qagtzero}[0]{{q_{\textnormal{a},0}}}
\newcommand{\qagt}[0]{{q_{\textnormal{a}}}}
\newcommand{\TargetSpace}[0]{{\mathcal{Q}_{\textnormal{tar}}}}

\newcolumntype{?}{!{\vrule width 1pt}}

\begin{document}

\title{Parallel, Asymptotically Optimal Algorithms for Moving Target Traveling Salesman Problems}

\author{Anoop Bhat$^{1}$, Geordan Gutow$^{2}$, Bhaskar Vundurthy$^{1}$,\\Zhongqiang Ren$^{3}$, Sivakumar Rathinam$^{4}$, and Howie Choset$^{1}$
\thanks{$^{1}$Robotics Institute at Carnegie Mellon University, 5000 Forbes Ave., Pittsburgh, PA 15213, USA. Emails: \{agbhat,
ggutow, pvundurt, choset\}@andrew.cmu.edu}%
\thanks{$^{2}$Mechanical and Aerospace Engineering at Michigan Technological University, Houghton, MI 49931. Email: gmgutow@mtu.edu}%
\thanks{$^{3}$Global College at Shanghai Jiao Tong University, Shanghai, China. Email: zhongqiang.ren@sjtu.edu.cn}%
\thanks{$^{4}$Department of Mechanical Engineering and Department of Computer Science and Engineering at Texas A\&M University, College Station, TX 77843. Email: srathinam@tamu.edu}%
}



\maketitle

\IEEEpubidadjcol

\begin{abstract}
The Moving Target Traveling Salesman Problem (MT-TSP) seeks a trajectory that intercepts several moving targets, within a particular time window for each target. When generic nonlinear target trajectories or kinematic constraints on the agent are present, no prior algorithm guarantees convergence to an optimal MT-TSP solution. Therefore, we introduce the Iterated Random Generalized (IRG) TSP framework. The idea behind IRG is to alternate between randomly sampling a set of agent configuration-time points, corresponding to interceptions of targets, and finding a sequence of interception points by solving a generalized TSP (GTSP). This alternation asymptotically converges to the optimum. We introduce two parallel algorithms within the IRG framework. The first algorithm, IRG-PGLNS, solves GTSPs using PGLNS, our parallelized extension of state-of-the-art solver GLNS. The second algorithm, Parallel Communicating GTSPs (PCG), solves GTSPs for several sets of points simultaneously. We present numerical results for three MT-TSP variants: one where intercepting a target only requires coming within a particular distance, another where the agent is a variable-speed Dubins car, and a third where the agent is a robot arm. We show that IRG-PGLNS and PCG converge faster than a baseline based on prior work. We further validate our framework with physical robot experiments.
\end{abstract}


\section{Introduction}\label{sec:intro}

\begin{figure}
    \centering
    \includegraphics[width=0.49\textwidth]{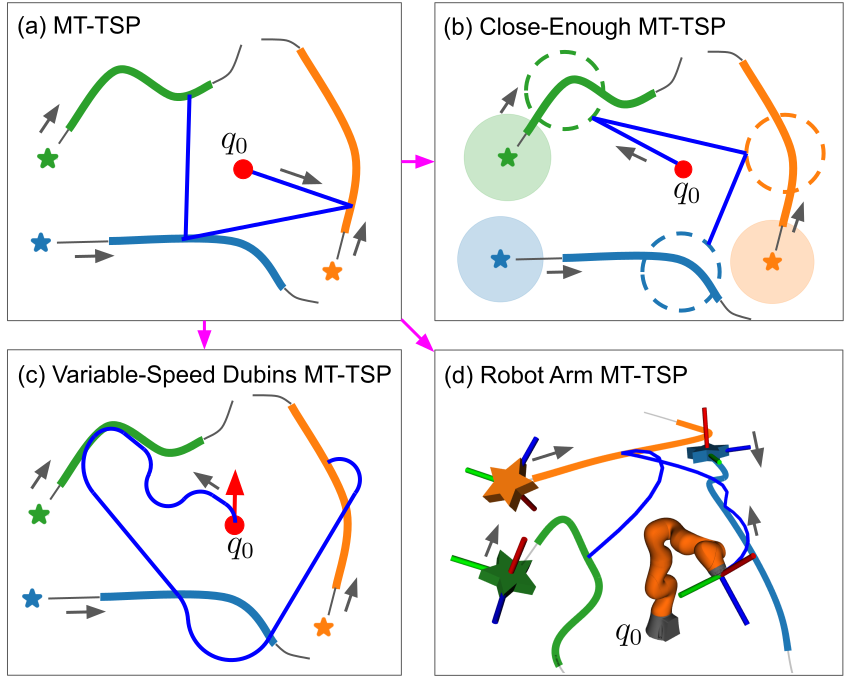}
    \vspace{-0.7cm}
    \caption{MT-TSP and three variants. In all images, targets (stars) move along trajectories with time windows shown in bold colored lines. Targets' locations are shown at $t = 0$. Agent's trajectory (blue) begins at initial configuration $q_0$ and intercepts all targets. We consider Hamiltonian paths, where the agent does not need to return to $q_0$. (a) MT-TSP, where the agent is limited by a maximum speed. (b) Close-Enough MT-TSP, where agent has a maximum speed, and each target is surrounded by a disc of positions where it may be intercepted. Filled discs are centered at the targets' positions at $t = 0$, and dashed, unfilled discs indicate disc locations at moment of interception. (c) Variable-Speed Dubins MT-TSP, where agent has a minimum speed, maximum speed, and a minimum turning radius that increases with speed. (d) Robot Arm MT-TSP. Agent is a 7-DOF arm (Kuka iiwa), where each joint has a speed limit. Intercepting a target requires matching the end-effector's pose with the target's pose. Poses are visualized at $t = 0$ with reference frames.}    
    \label{fig:intro_fig}
    \vspace{-0.6cm}
\end{figure}

The Traveling Salesman Problem (TSP) is a classic optimization problem with broad applications in various fields, including logistics, manufacturing, and robotics \cite{cook2011traveling, gutin2006traveling}. Given a set of targets (often called ``locations" or ``cities") and the cost of travel between each target pair, the TSP seeks a minimum-cost order of targets for an agent to visit. However, several robotic applications require planning to visit moving targets: midair refueling \cite{barnes2004solving}, optimization of fishing routes \cite{granado2024fishing,groba2015solving}, resupplying ships at sea \cite{brown2017scheduling}, surveillance \cite{marlow2007travelling,wang2023moving,de2019experimental}, and intercepting dangerous projectiles \cite{stieber2022DealingWithTime,smith2021assessment,helvig2003moving}. These applications lead to the Moving Target TSP (MT-TSP) \cite{helvig2003moving}, which is more challenging than the TSP, as it seeks to simultaneously optimize not only the visiting order of targets, but also a kinematically feasible agent trajectory that arrives at each target within a specified time window. Moreover, unlike in the TSP, the travel costs between targets are not fixed in the MT-TSP, but rather depend on the location in space and time at which the agent intercepts each target.

\IEEEpubidadjcol

In this paper, we address the MT-TSP, as well as three generalizations, shown in Fig. \ref{fig:intro_fig}:
\begin{itemize}
    \item \textbf{Close-Enough MT-TSP}: the agent needs only to pass through a disc centered on the moving target. This is motivated by scenarios such as wireless data transmission, where the disc represents the communication range.
    \item \textbf{Variable-Speed Dubins MT-TSP}: the agent has a minimum and maximum speed, as well as a minimum turning radius that increases with its speed. This is motivated by applications with ground vehicles and fixed-wing UAVs.
    \item \textbf{Robot Arm MT-TSP}: the agent is a redundant robotic arm that must intercept moving targets (e.g. parts on a conveyor belt) using its end effector.
\end{itemize}
All of these MT-TSP variants generalize the TSP, and thus finding optimal solutions is NP-hard \cite{hammar1999,helvig2003moving}. Furthermore, due to the presence of time windows, deciding whether a feasible solution even exists is NP-hard \cite{savelsbergh1985local}. To our knowledge, no prior work addresses any of the three MT-TSP variants above in the presence of time windows and nonlinear target trajectories. The closest related work is \cite{ding2022memetic}, which addresses the Close-Enough MT-TSP without time windows for a Dubins car. In this paper, we address all three variants in the presence of time windows using a single algorithmic framework.

\IEEEpubidadjcol

We call our framework Iterated Random Generalized (IRG) TSP. Algorithms within the IRG framework operate within one or more parallel processes. Each process performs several iterations, where each iteration constructs a \textit{sample point graph} containing a finite number of space-time points sampled along each target's trajectory. The sample point graph has clusters, where the points corresponding to a particular target form a cluster. An edge connects a point $s$ from one cluster to a point $s'$ in a different cluster if a transition from $s$ to $s'$ is kinematically feasible. For a robot whose only kinematic constraint is a speed limit, kinematically feasible simply means that there is enough time to travel from $s$ to $s'$. 

\IEEEpubidadjcol

After constructing a sample point graph, we solve a Generalized Traveling Salesman Problem (GTSP), which seeks a minimum-cost tour that visits exactly one point from each cluster. In subsequent iterations, we generate a new set of sample points that includes both randomly sampled points and points selected from previous GTSP solutions. We prove that our approach is asymptotically optimal, i.e. it converges asymptotically to an optimal solution through this iterative refinement.

\IEEEpubidadjcol

Algorithms in the IRG framework differ in the number of GTSPs solved simultaneously, and the method of solving the GTSP. Our first presented IRG algorithm, called IRG-PGLNS, solves one GTSP at a time, but uses a parallel GTSP solver, which we call Parallel Generalized Large Neighborhood Search (PGLNS). PGLNS extends the state-of-the-art GTSP solver Generalized Large Neighborhood Search (GLNS) \cite{smith2017glns} using ideas from parallel large neighborhood search \cite{ropke2009palns}. Our second presented IRG algorithm, called Parallel Communicating GTSPs (PCG), solves several GTSPs simultaneously, where each GTSP is solved using GLNS.

\IEEEpubidadjcol


Note that the sample point graphs in the MT-TSP are incomplete graphs, as it is often kinematically infeasible to travel between pairs of sampled points. Our analysis shows that state-of-the-art GTSP solvers, such as GLNS \cite{smith2017glns}, struggle to find feasible solutions on these incomplete graphs. To overcome this challenge, we present a new approach to finding a feasible GTSP solution, adapting pruning methods from the literature on the TSP with time windows \cite{Dumas1995OptimalAlgorithm}. This is critically important for the IRG framework to find its initial feasible solution, upon which it then iterates. Fig. \ref{fig:overall_framework_flowchart} summarizes the IRG framework's contributions.

\begin{figure}
    \centering
    \includegraphics[width=0.48\textwidth]{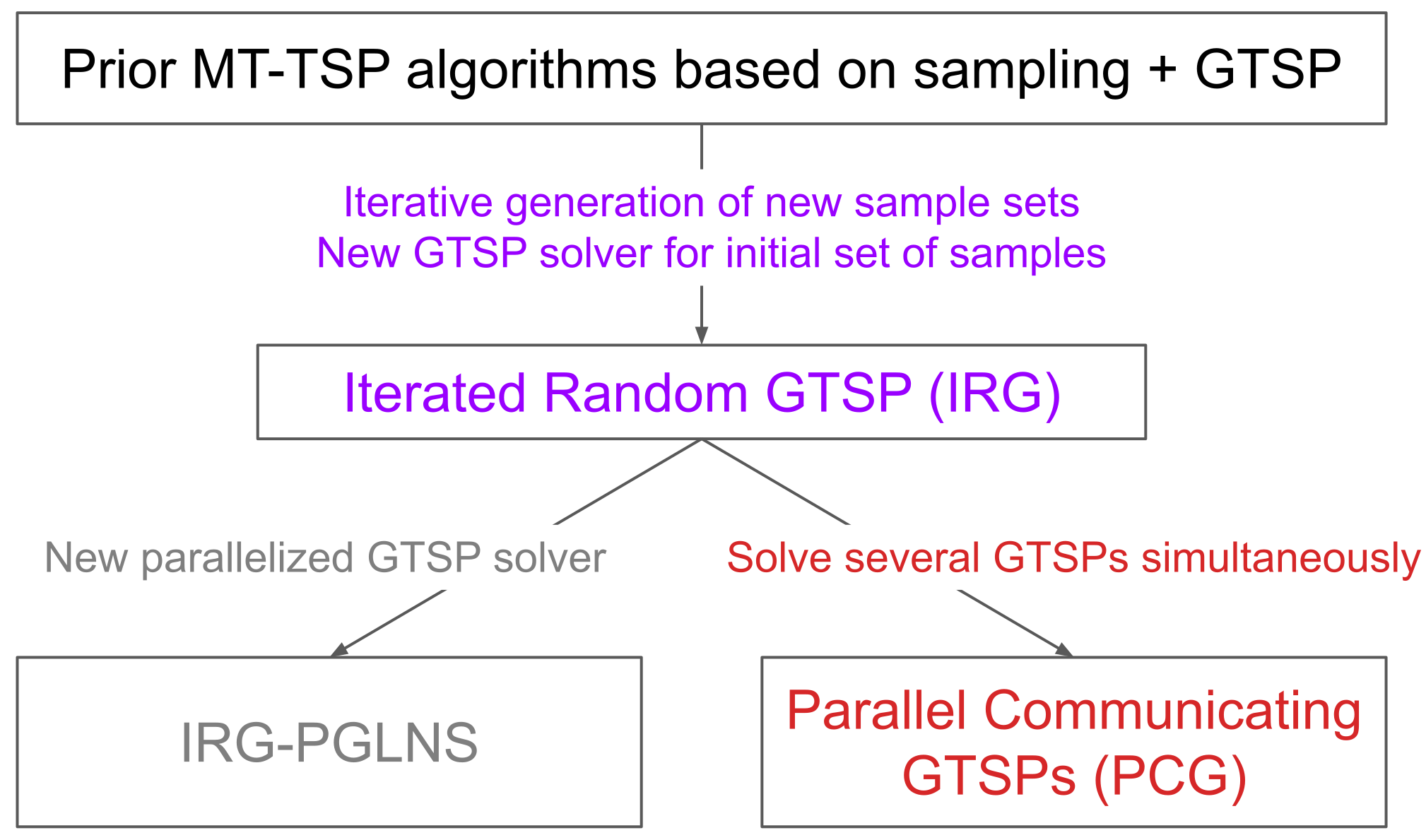}
    \vspace{-0.3cm}
    \caption{Contributions of IRG framework. As discussed in Section \ref{sec:mt_tsp}, prior MT-TSP methods exist that sample trajectories of targets into points, then solve a GTSP \cite{stieber2022DealingWithTime, philip2025CStar, Li2019RendezvousPlanning,mathew2015multirobot}. IRG improves upon these methods by iteratively generating new sample sets, contributing a new GTSP solver for the first set of samples, and accelerating convergence via two parallelization approaches.}
    \label{fig:overall_framework_flowchart}
    \vspace{-0.6cm}
\end{figure}

\IEEEpubidadjcol

We validate our methods with numerical experiments and experiments on a physical robot. The numerical experiments demonstrate that IRG-PGLNS and PCG converge faster than a baseline based on the memetic algorithm from \cite{ding2022memetic}. In addition, we demonstrate the importance of different components of the algorithms via ablation studies. We also show that PCG outperforms IRG-PGLNS for the Close-Enough MT-TSP, and that IRG-PGLNS outperforms PCG for the Variable-Speed Dubins MT-TSP. Finally, we show that our initial solution generation method, common to all IRG algorithms, finds feasible solutions faster than existing GTSP solution methods.

For our robot experiments, we track trajectories planned by the PCG algorithm, as well as baselines and ablations, on a mobile robot. We demonstrate that the executed trajectories from PCG tend to have smaller cost than the executed trajectories from the baselines and ablations.

\IEEEpubidadjcol

\section{Related Work}\label{sec:related_work}
\subsection{MT-TSP}\label{sec:mt_tsp}
Existing optimal algorithms\footnote{We say an algorithm is ``optimal" if it guarantees that it finds an optimal solution, and ``suboptimal" otherwise.} for the MT-TSP, based on mixed-integer programming (MIP), are limited to linear or piecewise-linear trajectories of targets and assume that the agent can move omnidirectionally \cite{philip2024mixedinteger, stieber2022DealingWithTime,stieber2022multiple,philip2025mixed}, while our work does not make these assumptions. \cite{philip2025CStar} lower bounds the optimal cost when the targets follow piecewise-linear trajectories or Dubins curves, and the objective is to minimize the agent's travel time. For finding feasible agent trajectories in the presence of generic nonlinear target trajectories and motion constraints on the agent, suboptimal algorithms exist~\cite{bourjolly2006orbit, stieber2022DealingWithTime, philip2025CStar, Li2019RendezvousPlanning,mathew2015multirobot, wang2023moving,choubey2013moving, de2019experimental,ding2022memetic}. For example, sampling-based algorithms \cite{stieber2022DealingWithTime, philip2025CStar, Li2019RendezvousPlanning,mathew2015multirobot} proceed by sampling the trajectories of targets into points in space-time, then finding a sequence of points to visit by solving a GTSP.
As we describe in the rest of this section, similar sampling-based methods exist for the Close-Enough TSP, Dubins TSP, and TSPs for robot arms. Our work extends these prior sampling-based methods. In particular, we incorporate resampling to achieve asymptotic optimality, as well as parallelization for speed.

Another suboptimal method for the MT-TSP is the memetic algorithm presented in \cite{ding2022memetic}, which addresses a Dubins Close-Enough MT-TSP without time windows. In our experiments, we adapt the method from \cite{ding2022memetic} to handle time windows and use it as a baseline. Finally, \cite{zhang2025two} provides a suboptimal method for a Close-Enough MT-TSP without time windows, but is specialized to linear target trajectories.

\subsection{Close-Enough TSP with Static Targets}\label{sec:close_enough_tsp}
The Close-Enough TSP \cite{gulczynski2006close} seeks the shortest path in space that visits a set of static targets, where visiting a target requires entering a disc centered at the target's location. Optimal algorithms have been developed for the Close-Enough TSP \cite{coutinho2016branch,zhang2023results}. The computational burden of optimal methods has motivated the development of suboptimal methods \cite{carrabs2017improved,carrabs2020adaptive,wang2019steiner,di2022genetic}. Our work on the Close-Enough MT-TSP extends the method from \cite{carrabs2017improved,carrabs2020adaptive}, which samples each target's disc into points in space, then solves a GTSP to find an order of targets to visit and a sampled point for each target. Other suboptimal approaches for the Close-Enough TSP include the variable neighborhood search in \cite{wang2019steiner} and the genetic algorithm in \cite{di2022genetic}.

\subsection{Dubins TSP with Static Targets}\label{sec:dubins_tsp}
The Dubins TSP \cite{savla2005point} seeks the shortest path in space that visits a set of stationary target positions such that the path satisfies a minimum turning radius constraint. The challenge in the Dubins TSP that distinguishes it from the classic TSP is determining the agent's heading angle at each target's position: if the heading at each target is given, then the cost of travel between each pair of targets is given by the length of a Dubins path \cite{dubins1957curves}, and the problem reduces to a classic TSP. While there are no optimal solvers for the Dubins TSP, methods have been developed to lower bound the optimal cost \cite{manyam2018tightly}. A common suboptimal approach is to sample a set of heading angles for each target, then solve a GTSP to select an order of targets and one of the sampled heading angles at each target \cite{ny2011dubins,cohen2017discretized,oberlin2010today,faigl2020fast}. Our work combines this approach with sampling-based methods for the MT-TSP. Another suboptimal approach for the Dubins TSP is the decoupled approach in \cite{ma2006receding}, which computes a sequence of targets using a Euclidean TSP and then optimizes the heading angles along the sequence. Alternatively, the approach in \cite{drchal2020wism} combines an evolutionary algorithm with machine learning.

Prior work \cite{kuvcerova2021variable} has also considered a Variable-Speed Dubins TSP where the agent's minimum turning radius is not fixed, but is instead a function of its speed, and there is an upper limit on the speed's derivative. To simplify the problem, \cite{kuvcerova2021variable} restricts the agent to move at a constant speed when turning and only accelerate when moving straight. Additionally, \cite{kuvcerova2021variable} only allows the agent to select from a finite set of speeds during constant-speed segments. \cite{wilson2025generalized} provides another model for a variable-speed Dubins car, which similarly selects from a finite set of speeds, but ignores limits on the speed's derivative. Our work adopts a model similar to \cite{wilson2025generalized}. Thus for a robot to track trajectories planned by our algorithms in application, additional techniques described in \cite{wilson2025generalized} may be needed to handle velocity continuity and acceleration constraints.

\subsection{Robot Arm TSP with Static Targets}\label{sec:robot_arm_tsp}
Variants of the TSP for a robot arm generally seek a minimum-cost path or trajectory through a robot arm's configuration space such that the end-effector pose passes through a given set of stationary target poses \cite{alatartsev2015robotic}. Different variants arise depending on whether the arm is redundant, as well as the constraints the arm must satisfy, e.g. joint speed limits \cite{abdel1990application}, joint torque limits \cite{dubowsky1989planning}, and/or obstacle avoidance \cite{wurll1999multi,saha2006planning,ruiz2018robotsp,zacharia2013task}. Our work considers a redundant arm with joint speed limits and no obstacles. One of the challenges in solving a TSP for a redundant arm is that along with finding an ordering of target end-effector poses, we must also select from the infinite number of arm configurations that can achieve each target pose. Our approach to handling redundancy in the Robot Arm MT-TSP extends \cite{saha2006planning}, which samples a set of arm configurations at each end-effector pose, then solves a GTSP to select a sequence of configurations. Another approach for handling redundancy for TSPs with robot arms is the decoupled approach from RoboTSP \cite{ruiz2018robotsp}, which finds a sequence of targets first by solving a TSP in end-effector space, then optimizes the arm configuration at each target with a sampling-based method, given the sequence. Additionally, \cite{zacharia2013task} applies a genetic algorithm to optimize the sequence of targets and arm configurations simultaneously.

\subsection{Parallel TSP Algorithms}
Parallel algorithms exist for the TSP \cite{verhoeven1995parallel, manfrin2006parallel, schneider1996searching, romanuke2024deep, fiechter1994parallel,ropke2009palns}, but they only provide means of optimizing an ordering of targets, whereas in the MT-TSP, we must also optimize the time and configuration at which the agent intercepts each target. Optimal MT-TSP algorithms \cite{philip2024mixedinteger, stieber2022DealingWithTime,stieber2022multiple,philip2025mixed} using MIP solvers such as Gurobi \cite{gurobi} can leverage the parallel capabilities of the MIP solver, but these optimal algorithms carry the limitations discussed in Section \ref{sec:mt_tsp}. Sampling-based MT-TSP algorithms using integer programming (IP) to solve their GTSP can similarly leverage their IP solver's parallel capabilities, but IP's computation time scales poorly with the number of targets, which we demonstrate in our experiments in Section \ref{sec:numerical_results}. Currently GLNS \cite{smith2017glns}, the most scalable method of solving GTSPs, is serial, and one of our contributions is the combination of parallel large neighborhood search \cite{ropke2009palns} with GLNS \cite{smith2017glns} in the new parallel GTSP solver PGLNS.

\SetKwFunction{RandomSamples}{RandomSamples}
\SetKwFunction{RandConfig}{RandConfig}
\SetKwFunction{TrajExists}{TrajExists}
\SetKwFunction{GetTraj}{GetTraj}
\SetKwFunction{MapSample}{MapSample}

\section{Problem Setup}\label{sec:problem_setup}

\subsection{Problem Statement}
We consider an agent with configuration space $\Qagt$, which is determined by the particular MT-TSP variant being considered. The agent has an initial configuration $\qagtzero$, which may correspond to a point robot that must get ``close enough" to a target, a Dubins car, or a robotic arm. Let the agent's trajectory be $\agenttrj : \mathbb{R}^+ \rightarrow \Qagt$. We assume that in each agent configuration $\qagt$, there is a set $\mathcal{A}_{\qagt}$ of admissible agent velocities, and we say that a trajectory $\agenttrj$ is \emph{kinematically feasible} in an interval $[t_0, t_f]$ if $\agenttrjdot(t) \in \mathcal{A}_{\agenttrj(t)}$ for all $t \in [t_0, t_f]$. We use $\mathcal{A}_{\qagt}$ to incorporate kinematic constraints, such as a minimum turning radius or a speed limit.

We denote the set of moving targets as $\mathcal{I} = \{1, 2, \dots, \ntar\}$, where $\ntar$ is the number of targets. Each target moves through a space $\TargetSpace$, which we define for each MT-TSP variant in Section \ref{sec:preliminaries}. Note that $\TargetSpace$ is the configuration space of an individual target, so the joint configuration space of all the targets is $\TargetSpace^{n_\text{tar}}$. The trajectory of target $i$ is $\tau_i: \mathbb{R}^+ \rightarrow \TargetSpace$. The time interval where target $i$ must be intercepted, i.e. the \emph{time window} of target $i$, is $[\underline{t}_i, \overline{t}_i] \subset \mathbb{R}^+$, where $\underline{t}_i$ is the start time and $\overline{t}_i$ is the end time. We assume that we are given an \emph{interception function} $\Xi_i: \Qagt \times \TargetSpace \rightarrow \{0, 1\}$, where if $\Xi_i(\qagt, \qtari) = 1$, the agent configuration $\qagt$ is said to \textit{intercept} the target configuration $q_{\text{tar}, i}$. If there exists $t_i \in [\underline{t}_i, \overline{t}_i]$ such that $\Xi_i(\agenttrj(t_i), \tau_i(t_i)) = 1$, we say that the agent trajectory $\tau_\text{a}$ intercepts target $i$. In the Close-Enough MT-TSP and Robot Arm MT-TSP, the cost function is the distance traveled in $\Qagt$. In the Variable-Speed Dubins MT-TSP, the cost function is the agent's distance traveled in $\mathbb{R}^2$.

\noindent \textbf{Problem:} The MT-TSP seeks a minimum-cost, kinematically feasible trajectory $\agenttrj$ that starts at $\qagtzero$ at $t = 0$ and intercepts all targets.

\subsection{Preliminaries}\label{sec:preliminaries}
IRG is a generic framework designed for problems that fit the description above. To apply IRG to a specific MT-TSP variant, we must specify $\Qagt$, $\TargetSpace$, $\mathcal{A}_{\qagt}$, and $\Xi_i$, as well as the following functions:
\begin{itemize}
    \item $\RandConfig_i$: given a target $i$ and time $t \in [\underline{t}_i, \overline{t}_i]$, $\RandConfig_i(t)$ outputs a random $\qagt \in \Qagt$ that intercepts target $i$ at time $t$. 
    
    \item $\TrajExists$: given some $\qagt \in \Qagt$, $t \in \mathbb{R}^+$, $\qagt' \in \Qagt$, and $t' \in \mathbb{R}^+$, $\TrajExists(\qagt, t, \qagt', t') = 1$ if a kinematically feasible agent trajectory $\agenttrj$ exists with $\agenttrj(t) = \qagt$ and $\agenttrj(t') = \qagt'$, and $\TrajExists(\qagt, t, \qagt', t') = 0$ if not.
    
    \item $\GetTraj$: given some $\qagt \in \Qagt$, $t \in \mathbb{R}^+$, $\qagt' \in \Qagt$, and $t' \in \mathbb{R}^+$, $\GetTraj(\qagt, t, \qagt', t')$ attempts to return a kinematically feasible agent trajectory $\agenttrj$ with $\agenttrj(t) = \qagt$ and $\agenttrj(t') = \qagt'$ and returns NULL if no such trajectory is found.
\end{itemize}
Next, we specify these sets and functions for the Close-Enough, Variable-Speed Dubins, and Robot Arm MT-TSPs.

\subsubsection{Close-Enough MT-TSP}
In the Close-Enough MT-TSP, $\Qagt = \TargetSpace = \mathbb{R}^2$ (recall that $\TargetSpace$ is the configuration space of an individual target). The agent has a maximum speed $v_\text{max}$. The set of admissible velocities at any configuration $\qagt$ is $\mathcal{A}_{\qagt} = \{\dot{q}_a : \|\dot{q}_a\| \leq v_\text{max}\}$. Each target $i$ is associated with a radius $r_i$, and $\Xi_i(\qagt, \qtari) = 1$ if and only if $\|\qagt - \qtari\| \leq r_i$. $\RandConfig_i(t)$ samples $\theta \in \mathbb{S}^1$ uniformly at random, then returns $\tau_i(t) + r_i\begin{bmatrix}\cos\theta\\\sin\theta\end{bmatrix}$, i.e. it randomly samples from the boundary of target $i$'s disc.\footnote{By only sampling on the disc boundaries, we sacrifice asymptotic optimality in degenerate problem instances where all optimal agent trajectories stay inside the disc of some target for the entirety of its time window.}

\TrajExists($\qagt, t, \qagt', t')$ checks if $\|\qagt - \qagt'\| \leq v_\text{max}(t' - t)$. \GetTraj($\qagt, t, \qagt', t')$ returns the constant-velocity trajectory from $(\qagt, t)$ to $(\qagt', t')$.

\subsubsection{Variable-Speed Dubins MT-TSP}\label{sec:variable_speed_dubins_problem_setup}
In the Variable-Speed Dubins MT-TSP, $\Qagt = SE(2)$ and $\TargetSpace = \mathbb{R}^2$. The agent, a variable-speed Dubins car \cite{wilson2025generalized}, has a minimum speed\footnote{When we refer to speed, curvature, or length of a trajectory or path for a variable-speed Dubins car, we refer to the position component of its trajectory or path, not the combination of the position and orientation components.} $v_\text{min}$, a maximum speed $v_\text{max}$, and a maximum turning rate $\omega_\text{max}$. While for a standard Dubins car, the minimum turning radius is fixed, a variable-speed Dubins car's minimum turning radius is proportional to its speed, i.e. for a speed $v$, the minimum turning radius is $\frac{v}{\omega_\text{max}}$. Consider an agent configuration $\qagt = (p, \phi)$ with $p \in \mathbb{R}^2$ and $\phi \in \mathbb{S}^1$. The velocity of the agent is $(\dot{p}, \dot{\phi})$. The set of admissible velocities at $\qagt$ is
\begin{align}
\begin{aligned}
 \mathcal{A}_{\qagt} = \{(\dot{p}, \dot{\phi}) : \; &\dot{p} = \|\dot{p}\|\begin{bmatrix}\cos\phi\\\sin\phi\end{bmatrix} \; \mbox{and} \; |\dot{\phi}| \in [0, \omega_\text{max}] \; \\&\mbox{and} \; \|\dot{p}\| \in [v_\text{min}, v_\text{max}]\}.
\end{aligned}
\end{align}
Consider a target configuration $\qtari \in \TargetSpace$. $\Xi_i(q, \qtari) = 1$ if and only if $p = q_{\text{tar}, i}$. $\RandConfig_i(t)$ samples a heading angle $\phi \in \mathbb{S}^1$ uniformly at random, then returns $\qagt = (\tau_i(t), \phi)$.

To define \TrajExists, we assume that the agent can only select from a finite set of speeds $\mathcal{S}_\text{speed} = \{v_1, v_2, \dots, v_{n_\text{speed}}\}$, where $n_\text{speed}$ is the number of allowed speeds. We additionally assume that between two consecutive interceptions of targets, the agent does not change its speed.
\TrajExists($\qagt, t, \qagt', t')$ returns 1 if and only if for some $v \in \mathcal{S}_\text{speed}$, a trajectory through $SE(2)$ exists from $(\qagt, t)$ to $(\qagt', t')$ with fixed speed $v$ and curvature no larger than $\frac{\omega_\text{max}}{v}$ (equivalently, turning radius no smaller than $\frac{v}{\omega_\text{max}}$); note that this check does not require a trajectory's curvature to be fixed. We check existence of such a trajectory using the methods from \cite{chen2023elongation}.

\GetTraj($\qagt, t, \qagt', t')$ iterates over the speeds in $\mathcal{S}_\text{speed}$ from smallest to largest, and for each speed $v$, attempts to generate a curvature-bounded trajectory from $(\qagt, t)$ to $(\qagt', t')$ with fixed speed $v$, using the methods from \cite{chen2023elongation}. \GetTraj returns the trajectory for the smallest $v$ where trajectory generation succeeds.

\subsubsection{Robot Arm MT-TSP}
In the Robot Arm MT-TSP, $\Qagt = [\underline{q}^1, \overline{q}^1] \times [\underline{q}^2, \overline{q}^2] \times \dots \times [\underline{q}^{\dim \Qagt}, \overline{q}^{\dim \Qagt}]$, where $[\underline{q}^j, \overline{q}^j]$ is the interval of allowable angles for joint $j$, and $\dim \Qagt$ is the number of joints in the arm. Additionally, $\TargetSpace = SE(3)$. Each joint $j$ has a speed limit $v_\text{max}^j$. The set of admissible velocities at any configuration $\qagt$ is $\mathcal{A}_{\qagt} = \{\dot{q}_a : |\dot{q}_a^j| \leq v_\text{max}^j \; \forall j\}$, where $\qagt^j$ is the $j$th element of $\qagt$. Let $FK: \Qagt \rightarrow SE(3)$ be the forward kinematic map. For $\qagt \in \Qagt$ and $\qtari \in \TargetSpace$, $\Xi_i(\qagt, \qtari) = 1$ if and only if $FK(\qagt) = \qtari$.

$\RandConfig_i(t)$ generates an inverse kinematics (IK) solution for the end-effector pose $\tau_i(t) \in SE(3)$. When there are multiple IK solutions, we have several options for redundancy resolution.
First, suppose we are using analytical IK. In several analytical IK solvers for redundant arms \cite{faria2018position, he2021analytical, singh2010analytical}, we can get a unique IK solution by specifying additional solver-specific parameters along with the end-effector pose. Thus, when using these analytical solvers, $\RandConfig_i(t)$ samples the solver-specific parameters uniformly at random, then passes these parameters along with $\tau_i(t)$ to the solver, obtaining a unique IK solution. If we instead use a numerical IK solver, we randomly sample an initial guess in $\Qagt$, pass the guess to the solver, and obtain a unique IK solution. Whether we use analytical or numerical IK, if the IK solution $\qagt$ does not fall within joint angle limits, $\RandConfig_i$ returns NULL.

\TrajExists($\qagt, t, \qagt', t')$ is implemented by checking if $|\qagt^j - \qagt'^j| \leq v_\text{max}^j(t' - t)$ for all $j \in \{1, 2, \dots, \dim\Qagt\}$. \GetTraj($\qagt, t, \qagt', t')$ returns the constant-velocity trajectory from $(\qagt, t)$ to $(\qagt', t')$.

\section{IRG-PGLNS Algorithm}\label{sec:irg_pglns}

\begin{figure}
    \centering
    \includegraphics[width=0.5\textwidth]{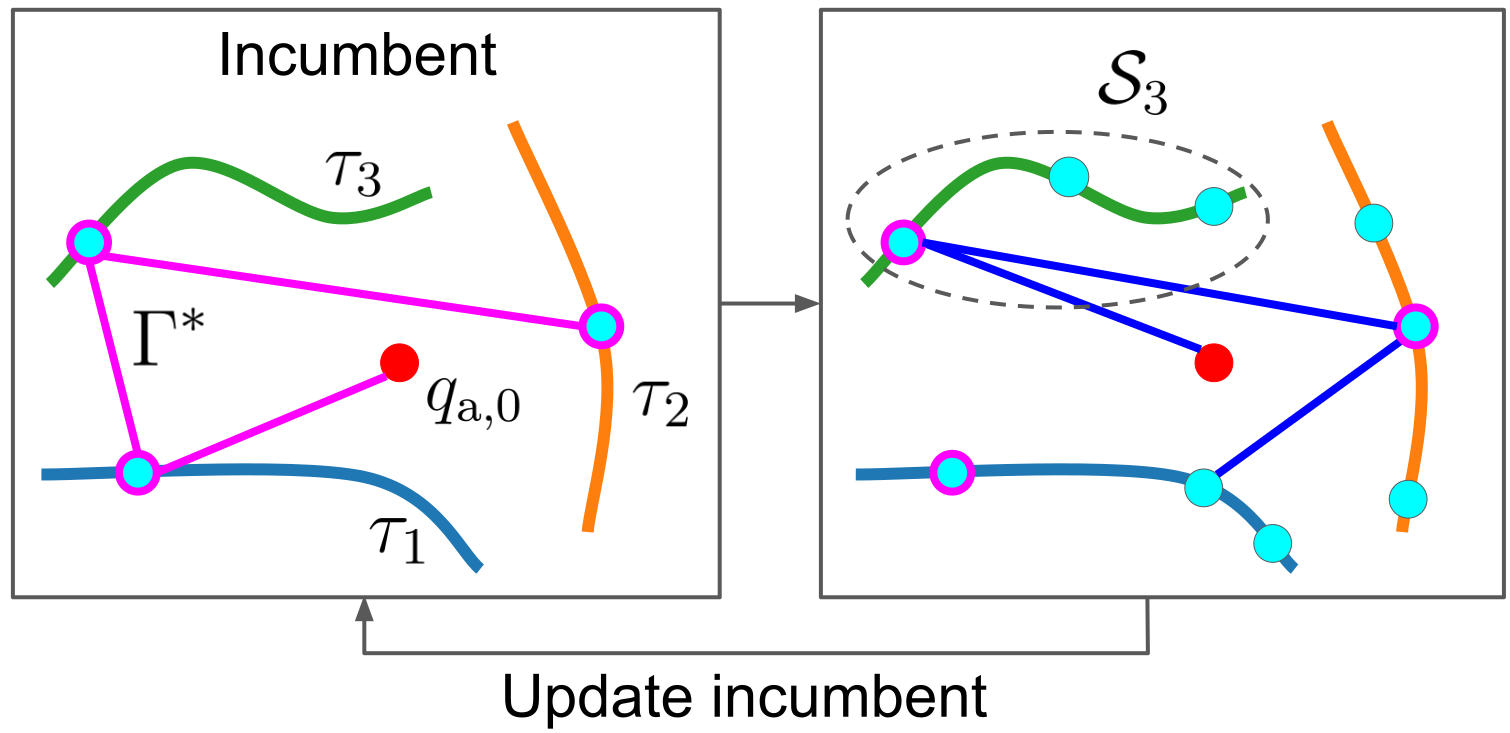}
    \vspace{-0.75cm}
    \caption{An iteration of tour improvement in IRG-PGLNS. The trajectory corresponding to the incumbent tour $\Gamma^*$ (the least-cost tour found so far) is shown in pink. Points in the incumbent are outlined in pink. To improve the incumbent, as seen in the right-hand box, we generate a set of sample points $\mathcal{S}_i$ for each target $i$, including random points and the point from the incumbent. We then solve a GTSP to find an updated incumbent.}
    \label{fig:irg_alg_expl_fig}
    \vspace{-0.3cm}
\end{figure}

\SetKwFunction{GenerateInitialTour}{GenerateInitialTour}
\SetKwFunction{TourViaGTSP}{TourViaGTSP}

In this section, we present our first instantiation of the IRG framework, which we call IRG-PGLNS, due to its use of our novel GTSP solver, PGLNS. IRG-PGLNS searches the space of \emph{tours}, where a tour is a sequence of points (configuration-time pairs) in $\Qagt \times \mathbb{R}^+$ beginning with $(\qagtzero, 0)$, such that each target is intercepted by one point in the sequence. Upon termination, IRG-PGLNS uses \GetTraj to connect each pair of consecutive points in its best tour $\Gamma^*$ to form a trajectory.\footnote{In this work, we do not require continuity of the trajectory's velocity.} We search the space of tours because in the Variable-Speed Dubins MT-TSP, checking if a trajectory exists between two points (i.e. invoking \TrajExists) is computationally cheaper than computing the trajectory using \GetTraj. We leverage the ability to compute a tour's cost without computing the associated trajectory.

\SetKwFunction{GetInterceptionPoint}{GetInterceptionPoint}

IRG-PGLNS is described by Alg. \ref{alg:IRG_PGLNS} and illustrated in Fig. \ref{fig:irg_alg_expl_fig}. Throughout IRG-PGLNS, we refer to the best tour we have found so far as the \emph{incumbent}, denoted as $\Gamma^*$. IRG-PGLNS initializes $\Gamma^*$ (Alg. \ref{alg:IRG_PGLNS}, Line \ref{algline:irg_gen_init_tour}) using the \GenerateInitialTour procedure, detailed in Section \ref{subsec:init_tour_gen}. IRG-PGLNS then begins its tour improvement loop (Line \ref{algline:irg_tour_improvement_loop}). Each iteration of this loop generates a set of sample points $\mathcal{S}_i$ for each target $i$. $\mathcal{S}_i$ contains the point where $\Gamma^*$ visits target $i$, denoted by $\GetInterceptionPoint(i, \Gamma^*)$, as well as $n_\text{rand}$ points randomly sampled using Alg. \ref{alg:RandomSamples}. Alg. \ref{alg:RandomSamples} repeatedly samples a time $t$ in target $i$'s time window at random, then samples an agent configuration $\qagt$ using $\RandConfig_i$, to obtain a random point $(\qagt, t)$. While Fig. \ref{fig:irg_alg_expl_fig} shows the sampling strategy for the MT-TSP, Fig. \ref{fig:sampling_for_three_variants} shows the sampling strategy for the Close-Enough, Variable-Speed Dubins, and Robot Arm MT-TSPs.

\begin{figure}
    \centering
    \includegraphics[width=0.5\textwidth]{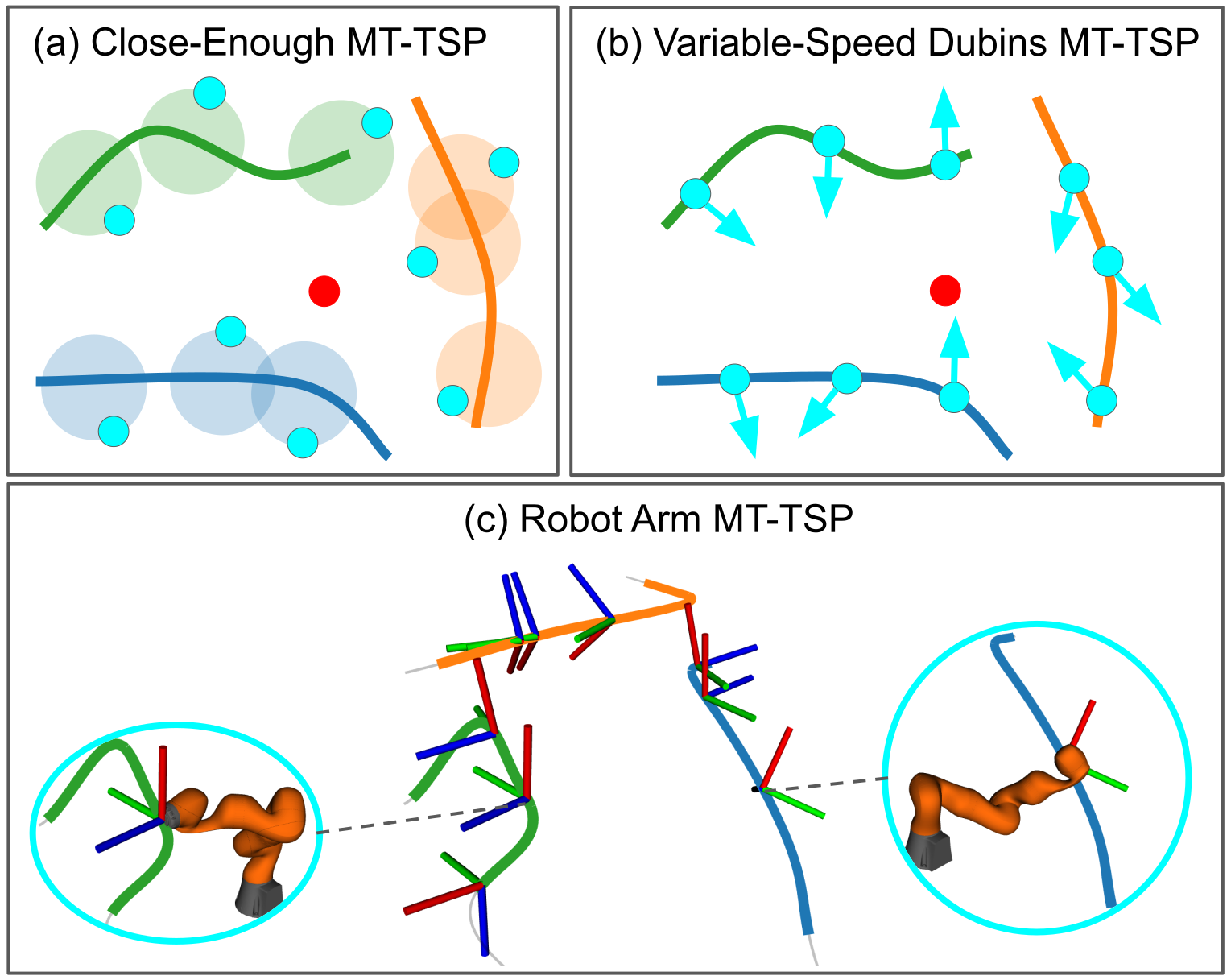}
    \vspace{-0.75cm}
    \caption{IRG-PGLNS sampling strategy applied to three MT-TSP variants. (a) Close-Enough MT-TSP: samples lie on disc boundaries. (b) Variable-Speed Dubins MT-TSP: a sample point's position matches the target's position at the sampled time, but the sample point's heading is random. (c) Robot Arm MT-TSP: each sample point is an arm configuration-time pair, where the end-effector pose for the arm configuration must match the target's pose at the sampled time. Target poses at sampled times are shown with reference frames.}
    \label{fig:sampling_for_three_variants}
    \vspace{-0.3cm}
\end{figure}

After constructing $\mathcal{S}_i$ for all targets, IRG-PGLNS finds a new incumbent $\Gamma^*$ using \TourViaGTSP, described in Section \ref{subsec:tour_via_gtsp}. We pass the current $\Gamma^*$ as a seed tour that \TourViaGTSP improves upon. When planning time runs out, IRG-PGLNS returns the trajectory associated with $\Gamma^*$.

\SetKwFunction{UniformRandomSample}{UniformRandomSample}
\SetKwFunction{ConstructGTSPGraph}{ConstructGTSPGraph}
\SetKwFunction{DepthFirstSearch}{DepthFirstSearch}

\begin{algorithm}\label{alg:IRG_PGLNS}
\SetKwFunction{IRGPGLNS}{IRG-PGLNS}
\caption{\protect\IRGPGLNS}
$\Gamma^* = $ \GenerateInitialTour()\;\label{algline:irg_gen_init_tour}
\While{\upshape Time limit has not been reached}{\label{algline:irg_tour_improvement_loop}
    \For{\upshape $i \in \{1, 2, \dots, \ntar\}$}{\label{algline:irg_sampling_loop}
        $\mathcal{S}_i = \{\GetInterceptionPoint(i, \Gamma^*)\}\cup\RandomSamples(i, n_\text{rand})$\;
    }
    $\Gamma^* = \TourViaGTSP(\{\mathcal{S}_i\}_{i \in \mathcal{I}}, \Gamma^*)$\;\label{algline:irg_end_tour_improvement_loop}
}

Return trajectory associated with $\Gamma^*$\;\label{algline:irg_return_trj}
\end{algorithm}
\vspace{-0.7cm}

\subsection{Initial Tour Generation}\label{subsec:init_tour_gen}
\GenerateInitialTour, described by Alg. \ref{alg:InitialTourGen}, begins by randomly sampling a set of points $\mathcal{S}_i$ for each target $i$ (Lines \ref{algline:init_tour_sampling_loop}-\ref{algline:init_tour_sampling}). Alg. \ref{alg:InitialTourGen} then seeks a tour $\Gamma^*$ visiting each target $i$ at a point in $\mathcal{S}_i$ via \TourViaGTSP (Section \ref{subsec:tour_via_gtsp}). In contrast to Alg. \ref{alg:IRG_PGLNS}, we do not pass a seed tour to \TourViaGTSP, since we do not have a seed tour to pass. For a given set of sample points, \TourViaGTSP may not find a feasible tour. If \TourViaGTSP finds a feasible tour $\Gamma^*$, Alg. \ref{alg:InitialTourGen} returns $\Gamma^*$ (Line \ref{algline:return_init_feas_tour}). Otherwise, if there is time left, Alg. \ref{alg:InitialTourGen} adds more samples to each $\mathcal{S}_i$ and attempts to find a tour again.

\begin{algorithm}\label{alg:RandomSamples}
\caption{\protect\RandomSamples($i, n$)}
$\mathcal{S} = \emptyset$\;
\For{\upshape $k \in \{1, 2, \dots, n\}$}{
    $\qagt = $ NULL\;
    \While{\upshape $\qagt$ is NULL}{
        $t = $ \UniformRandomSample($[\underline{t}_i, \overline{t}_i]$)\;
        $\qagt = \RandConfig_i(t)$\;
    }
    $\mathcal{S} = \mathcal{S}\cup\{(\qagt, t)\}$\;
}
return $\mathcal{S}$\; 
\end{algorithm}
\vspace{-0.5cm}

\begin{algorithm}\label{alg:InitialTourGen}
\caption{\protect\GenerateInitialTour}
$\mathcal{S}_i = \emptyset$ for all $i \in \mathcal{I}$\;
\While{\upshape Time limit has not been reached}{\label{algline:init_tour_main_loop}
    \For{\upshape $i \in \{1, 2, \dots, \ntar\}$}{\label{algline:init_tour_sampling_loop}
        $\mathcal{S}_i = \mathcal{S}_i\cup\RandomSamples(i, n_\text{rand, init})$\;\label{algline:init_tour_sampling}
    }
    $\Gamma^* = \TourViaGTSP(\{\mathcal{S}_i\}_{i \in \mathcal{I}})$\;
    \lIf{\upshape $\Gamma^*$ is not NULL}{return $\Gamma^*$\label{algline:return_init_feas_tour}}
}
return NULL\;
\end{algorithm}
\vspace{-0.5cm}

\subsection{Finding Tour via GTSP}\label{subsec:tour_via_gtsp}
The \TourViaGTSP function, used in Alg. \ref{alg:IRG_PGLNS} and Alg. \ref{alg:InitialTourGen}, seeks a tour where the point intercepting target $i$ is an element of $\mathcal{S}_i$. We first construct a \emph{sample point graph} $\mathcal{G} = (\mathcal{V}, \mathcal{E})$, where $\mathcal{V}$ is the set of nodes and $\mathcal{E}$ is the set of edges. $\mathcal{V}$ contains all sample points and the pairing of $\qagtzero$ with time $0$, i.e. $\mathcal{V} = \bigcup\limits_{i \in \mathcal{I}}\mathcal{S}_i \cup \{(\qagtzero, 0)\}$. An edge connects node $(\qagt, t)$ to node $(\qagt', t')$ if \TrajExists($\qagt, t, \qagt', t'$) = 1 and the nodes belong to different sets $\mathcal{S}_i$. The edge cost depends on the problem variant. In the Close-Enough and Robot Arm MT-TSPs, the edge cost is $\|\qagt' - \qagt\|$. In the Variable-Speed Dubins MT-TSP, the edge cost is $v^*(\qagt, t, \qagt', t')(t' - t)$, where $v^*(\qagt, t, \qagt', t')$ is the minimum $v \in \mathcal{S}_\text{speed}$ such that the trajectory existence check performed by \TrajExists succeeds (this check is described in Section \ref{sec:variable_speed_dubins_problem_setup}). We can find a tour by finding a sequence of nodes in $\mathcal{G}$ beginning at $(\qagtzero, 0)$ and containing one node in each $\mathcal{S}_i$. This is the same as solving a GTSP on $\mathcal{G}$.

IRG-PGLNS solves the GTSP using one of two methods. If no seed tour is passed, we use the depth-first search (DFS) described in Section \ref{subsec:dfs}. If a seed tour is passed, IRG-PGLNS solves the GTSP using PGLNS, which is described in Section \ref{subsec:pglns}. We only run PGLNS if we have a seed tour because we have found that on incomplete graphs such as $\mathcal{G}$, PGLNS and its predecessor GLNS struggle to find a feasible tour unless we provide one, as shown in Section \ref{sec:mt_tsp_eval_init_tour_gen}.

PGLNS, like its predecessor GLNS, requires a matrix of integer edge costs, where the entry for row $i$, column $j$ is the cost of edge $(i, j)$. Note that $\mathcal{G}$ is an incomplete graph, i.e. there are some edges $(i, j)$ that do not exist in $\mathcal{G}$. We call such nonexistent edges \emph{infeasible} edges, as opposed to \emph{feasible} edges that actually exist in $\mathcal{G}$. To interface with PGLNS, we still need to populate costs for infeasible edges. To populate the cost matrix, we first multiply all feasible edge costs by 100, round to the nearest integer, and populate the corresponding entries in the cost matrix, effectively rounding all edge costs to two decimal places. For all infeasible edges, we need the cost to be large enough that any tour traversing an infeasible edge will have larger cost than the incumbent; thus, we set the cost equal to the cost of the current incumbent, computed using scaled and rounded edge costs, plus one.

To guarantee asymptotic optimality despite using the suboptimal GTSP solver PGLNS, \TourViaGTSP does not immediately return the tour found by PGLNS. Instead, after obtaining a tour $\Gamma^\text{tmp}$ from PGLNS, we also generate a random tour $\Gamma^{\text{rand}}$ using the current sets $\mathcal{S}_i$. If $\Gamma^\text{rand}$ has lower cost than $\Gamma^\text{tmp}$, \TourViaGTSP returns $\Gamma^\text{rand}$; otherwise, \TourViaGTSP returns $\Gamma^\text{tmp}$. This extra step of generating $\Gamma^\text{rand}$ is only of theoretical value, as $\Gamma^\text{rand}$ is almost never chosen over $\Gamma^\text{tmp}$ in practice.

\subsection{Solving GTSP via Depth-First Search}\label{subsec:dfs}
\SetKwFunction{Tar}{Tar}

When generating the initial tour, we solve the GTSP in Section \ref{subsec:tour_via_gtsp} using the DFS in Alg. \ref{alg:InitialDFS}, aiming to quickly find a feasible solution without considering its cost. Before beginning the main loop, we construct a set BEFORE[$s$] for each node $s \in \mathcal{V}$ (Lines \ref{algline:init_before}-\ref{algline:done_init_before}), containing all targets that have no sample points reachable from $s$, i.e. all targets that must be visited before $s$. The construction of BEFORE is inspired by \cite{Dumas1995OptimalAlgorithm}, which addresses the TSP with time windows.

After constructing BEFORE, we initialize a stack of \emph{search nodes} (Line \ref{algline:dfs_init_stack}), where a search node is a tuple $u = (\mathcal{U}, s)$, with $\mathcal{U} \subseteq \mathcal{I}$ and $s \in \mathcal{V}$. A search node represents the set of tours through $\mathcal{G}$ visiting the same set of targets $\mathcal{U}$ (perhaps in different orders) and terminating with the same sample point $s \in \mathcal{V}$. We maintain a backpointer for each search node $u$, denoted as $u.\text{bp}$, equal to a search node previously popped from the stack. Once we set a backpointer for a search node, the search node represents a particular tour, which we can reconstruct by traversing the backpointer chain. We also use a CLOSED set to avoid re-expanding search nodes.

Each iteration of the main search loop pops a search node $u = (\mathcal{U}, s)$ from the stack, and if $\mathcal{U} = \mathcal{I}$, the search terminates and returns a tour reconstructed from $u$ (Lines \ref{algline:dfs_pop}-\ref{algline:dfs_reconstruct}). If $\mathcal{U} \neq \mathcal{I}$, we generate the \emph{successor $\mathcal{G}$-nodes} of $u$ (Line \ref{algline:dfs_successors}), which are the nodes $s' \in \mathcal{V}$ satisfying the following conditions:
\begin{enumerate}
    \item $(s, s') \in \mathcal{E}$.\label{cond:valid_tour}

    \item $s' \notin \mathcal{S}_i \; \forall \; i \in \mathcal{U}$, ensuring each target is visited once.\label{cond:visit_each_target_once}

    \item BEFORE[$s'$] $\subseteq \mathcal{U}$, ensuring that by visiting $s'$, we do not prevent any unvisited targets from being visited.\label{cond:before_cond}
\end{enumerate}
Finally, for each successor $\mathcal{G}$-node $s'$, we generate a successor search node $u' = (\mathcal{U}\cup \{\Tar(s')\}, s')$, where $\Tar(s')$ is the target of $s'$, then push $u'$ onto the stack (Lines \ref{algline:dfs_successor_search_node}-\ref{algline:dfs_push}). We push successors onto the stack in order of decreasing edge cost from $s$ to $s'$, such that the least-cost successor gets popped from the stack next. The search nodes and successor relationships in Alg. \ref{alg:InitialDFS} implicitly form a directed acyclic graph (DAG), so we refer to Alg. \ref{alg:InitialDFS} as DAG-DFS in the experiments.

\begin{algorithm}\label{alg:InitialDFS}
\caption{\protect\DepthFirstSearch}
\SetKwFunction{ReconstructTour}{ReconstructTour}

BEFORE = dict()\;\label{algline:init_before}
\For{\upshape $s \in \mathcal{V}$}{
    BEFORE[$s$] = $\{i \in \mathcal{I} : s \notin \mathcal{S}_i \; \mbox{and} \; (\forall s' \in \mathcal{S}_i)((s, s') \notin \mathcal{E}) \}$\;
}\label{algline:done_init_before}

STACK = [$(\emptyset, (\qagtzero, 0))$]\;\label{algline:dfs_init_stack}
CLOSED = $\emptyset$\;\label{algline:dfs_init_closed}
\While{\upshape STACK is not empty}{
    $u = (\mathcal{U}, s) = $ STACK.pop()\;\label{algline:dfs_pop}
    \lIf{\upshape $u \in $ CLOSED}{continue} 
    CLOSED.insert($u$)\;
    \lIf{\upshape $\mathcal{U} = \mathcal{I}$}{return \ReconstructTour($u$)\label{algline:dfs_reconstruct}}
    \For{\upshape $s'$ in $u$.successor$\mathcal{G}$Nodes()}{\label{algline:dfs_successors}
        $u' = (\mathcal{U} \cup \{\Tar(s')\}, s')$\;\label{algline:dfs_successor_search_node}
        \lIf{\upshape $u' \in $ CLOSED}{continue} 
        $u'.$bp = $u$\;
        STACK.push($u'$)\;\label{algline:dfs_push}
    }
}
return NULL\;
\end{algorithm}
\vspace{-0.7cm}

\subsection{PGLNS}\label{subsec:pglns}
When a seed tour $\Gamma^*$ is passed to \TourViaGTSP, IRG-PGLNS solves the GTSP using PGLNS, which is described by Alg. \ref{alg:PGLNS}. Like its predecessor GLNS, PGLNS runs $n_\text{warm}$ ``warm" trials. Each warm trial begins by initializing a current tour $\Gamma^\text{c}$, a counter $i_\text{term}$, a scalar temperature $\theta$, a cooling rate $r_\text{cool}$, and several locks (Lines \ref{algline:pglns_init} to \ref{algline:pglns_init_done}). The determination of the temperature and cooling rate using the warm trial counter is explained in Section 5.3 of \cite{smith2017glns}. The warm trial then runs several threads, each of which runs an inner loop (Line \ref{algline:pglns_thread_inner_loop}).

Each iteration of a thread $j$'s inner loop copies $\Gamma^{\text{c}, j}$, removes several nodes from $\Gamma^{\text{c}, j}$, then inserts nodes back into $\Gamma^{\text{c}, j}$ (Line \ref{algline:pglns_remove_insert}). The nodes to remove are selected using a removal heuristic, and the nodes to insert are selected using an insertion heuristic. The heuristics are randomly sampled from the same bank of heuristics used in \cite{smith2017glns}. If the modified $\Gamma^{\text{c}, j}$ is accepted (determined using the simulated annealing criteria from \cite{smith2017glns}), thread $j$ updates $\Gamma^\text{c}$ to $\Gamma^{\text{c}, j}$ (Lines \ref{algline:pglns_copy_theta}-\ref{algline:pglns_update_Gamma_c_1}). Then, if $\Gamma^{\text{c}, j}$ has lower cost than $\Gamma^*$, thread $j$ updates $\Gamma^*$ (Lines \ref{algline:pglns_lock_lstar_1}-\ref{algline:pglns_update_Gamma_star_1}). Thread $j$ then performs the local optimizations described in \cite{smith2017glns} on $\Gamma^{\text{c}, j}$, and if it is still better than $\Gamma^*$ (i.e. if another thread did not find a better tour), thread $j$ updates $\Gamma^*$ again (Lines \ref{algline:pglns_local_opt}-\ref{algline:pglns_update_Gamma_star_2}).

Every time we update $\Gamma^*$, we reset $i_\text{term}$ (Line \ref{algline:pglns_reset_i_term}), and if we do not update $\Gamma^*$ in the current iteration, we increment $i_\text{term}$ (Line \ref{algline:pglns_increment_i_term}). A thread terminates its inner loop when $i_\text{term}$ reaches a threshold (Lines \ref{algline:pglns_lock_lterm}-\ref{algline:pglns_thread_break}), which is a smaller number if $\Gamma^*$ has never been improved in the current warm trial, and a larger number thereafter, as in \cite{smith2017glns}. The $n_\text{term}$ parameter corresponds to the num\_iterations parameter in GLNS, and the thresholds $n_\text{term}/6$ and $n_\text{term}/4$ are from GLNS's fast mode.

\SetKwFunction{PGLNS}{PGLNS}
\SetKwFunction{LockAndCopy}{LockAndCopy}
\SetKwFunction{RemoveAndInsert}{RemoveAndInsert}
\SetKwFunction{Lock}{Lock}
\SetKwFunction{Copy}{Copy}
\SetKwFunction{Accept}{Accept}
\SetKwFunction{Cost}{Cost}
\SetKwFunction{TermThreshold}{Threshold}
\SetKwFunction{SetTemperatureAndCoolingRate}{SetTempAndCoolingRate}
\SetKwBlock{Do}{}{end}
\SetKwFor{ParFor}{parallel for}{do}{end}
\begin{algorithm}\label{alg:PGLNS}
\caption{\protect\PGLNS($\Gamma^*$)}
\For{\upshape $i_\text{warm} = 1, 2, \dots, n_\text{warm}$}{
    $\Gamma^\text{c} = \Copy(\Gamma^*), i_\text{term} = 1$, improved = false\;\label{algline:pglns_init}
    $\theta, r_\text{cool} = \SetTemperatureAndCoolingRate(i_\text{warm})$\;
    Initialize locks $l^\text{c}$, $l^*$, $l^{\text{term}}$, $l^\theta$\;\label{algline:pglns_init_done}
    \ParFor{\upshape $j = 1, 2, \dots, n_\text{thread}$}{
        \While{\upshape true}{\label{algline:pglns_thread_inner_loop}
            With \Lock($l^\text{term}$) \Do{\label{algline:pglns_lock_lterm}
                \lIf{\upshape (not improved and $i_\text{term} > n_\text{term}/6$) or (improved and $i_\text{term} > n_\text{term}/4$) or other thread broke}
                {
                    break\label{algline:pglns_thread_break}
                }
            }
            With $\Lock(l^\text{c})$, $\Gamma^{\text{c}, j} = \Copy(\Gamma^\text{c})$\;
            $\Gamma^{\text{c}, j} = \RemoveAndInsert(\Gamma^{\text{c}, j})$\;\label{algline:pglns_remove_insert}
            With $\Lock(l^\theta$), $\theta^j = \Copy(\theta)$\;\label{algline:pglns_copy_theta}
            With $\Lock(l^\text{c})$ accept = \Accept($\Gamma^{\text{c}, j}, \Gamma^\text{c}, \theta^j$)\;\label{algline:pglns_accept}
            \uIf{\upshape accept}{
                With $\Lock(l^\text{c})$, $\Gamma^\text{c} = \Gamma^{\text{c}, j}$\;\label{algline:pglns_update_Gamma_c_1}
        
                With $\Lock(l^*)$ \Do{\label{algline:pglns_lock_lstar_1}
                    update\_best = $\Cost(\Gamma^{\text{c}, j}) < \Cost(\Gamma^*)$\;
                    \lIf{\upshape update\_best}{
                        $\Gamma^* = \Copy(\Gamma^{\text{c}, j})$\label{algline:pglns_update_Gamma_star_1}
                    }
                }
                \If{\upshape update\_best}{
                    With $\Lock(l^\text{term})$ \Do{
                      $i_\text{term} = 1$\;\label{algline:pglns_reset_i_term}
                      improved = true\;
                    }
                    
                    Locally reoptimize $\Gamma^{\text{c}, j}$\label{algline:pglns_local_opt}
                    
                    With $\Lock(l^*)$ \Do{
                        \If{\upshape $\Cost(\Gamma^{\text{c}, j}) < \Cost(\Gamma^*)$}{
                            $\Gamma^* = \Copy(\Gamma^{\text{c}, j})$\;\label{algline:pglns_update_Gamma_star_2}
                        }
                    }
                    
                    With $\Lock(l^\text{c})$ \Do{
                        \If{\upshape $\Cost(\Gamma^{\text{c}, j}) < \Cost(\Gamma^\text{c})$}{
                            $\Gamma^c = \Copy(\Gamma^{\text{c}, j})$\;
                        }
                    }
                }
                \Else{
                    With $\Lock(l^\text{term})$ $i_\text{term} = i_\text{term} + 1$\;\label{algline:pglns_increment_i_term}
                }
            }
            With $\Lock(l^\theta)$ $\theta = \theta r_\text{cool}$
        }
    }
}
return $\Gamma^*$\;
\end{algorithm}

\begin{figure*}
    \centering
    \includegraphics[width=\textwidth]{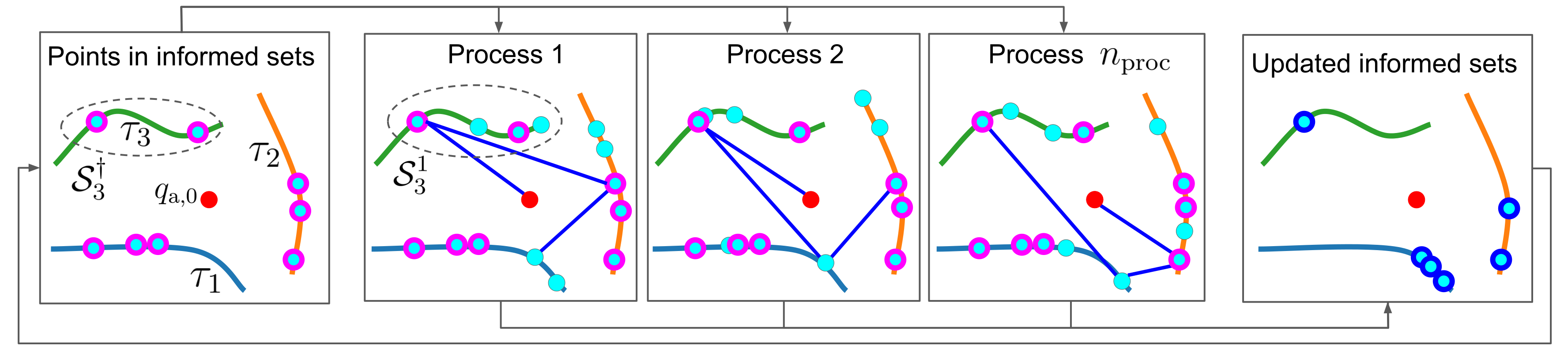}
    \vspace{-0.7cm}
    \caption{Illustration of an iteration of trajectory improvement in PCG. The main process maintains an informed set of points $\mathcal{S}_i^\dagger$ for each target $i$. Points in the current iteration's informed sets are outlined in pink. Each child process $j$ generates a set of sample points $\mathcal{S}_i^j$ for each target $i$, containing the points from the current $\mathcal{S}_i^\dagger$ and random points unique to process $j$. Each child process then solves a GTSP, finding a sequence of points beginning at $\qagtzero$ visiting one point per target. The main process then updates the informed set for each target $i$ to contain all visited points associated with target $i$ from all child processes.}
    \label{fig:pcg_alg_expl_fig}
    \vspace{-0.6cm}
\end{figure*}

\SetKwFunction{GenerateInitialTour}{GenerateInitialTour}
\SetKwFunction{TourViaGTSP}{TourViaGTSP}

\section{PCG}\label{sec:pcg}
We now present our second instantiation of the IRG framework, which we call PCG. PCG shares many of the subroutines of IRG-PGLNS, but employs a different form of parallelization. Instead of using the parallel solver PGLNS to solve GTSPs, PCG uses the serial solver GLNS, but solves several GTSPs simultaneously, each corresponding to a different set of sample points. PCG is described by Alg. \ref{alg:PCG}, which we call the main process. Fig. \ref{fig:pcg_alg_expl_fig} shows an illustration. PCG begins by finding an initial tour $\Gamma^*$ using the same method as Alg. \ref{alg:IRG_PGLNS}. Next, PCG initializes an \emph{informed set} $\mathcal{S}_i^\dagger$ for each target. In particular, it initializes $\mathcal{S}_i^\dagger$ as a singleton set, containing the point $(\qagt, t) \in \Qagt \times \mathbb{R}^+$ at which $\Gamma^*$ visits target $i$ (Line \ref{algline:pcg_init_informed_sets_after_dfs}).

PCG then begins its tour improvement loop (Line \ref{algline:pcg_tour_improv_loop}). Each iteration of this loop runs $n_\text{proc}$ parallel child processes. Each child process $j$ generates a set of sample points $\mathcal{S}_i^j$ for each target $i$ (Lines \ref{algline:pcg_child_proc_loop_begin}-\ref{algline:pcg_sampling}). $\mathcal{S}_i^j$ contains all points in $\mathcal{S}_i^\dagger$, as well as $n_\text{rand}$ randomly sampled points unique to process $j$. After constructing $\mathcal{S}_i^j$, each process $j$ finds a tour $\Gamma^j$ by solving a GTSP (Line \ref{algline:pcg_sampling}). Rather than using PGLNS at this step, we use the serial GLNS. We found that if we set $n_\text{proc}$ equal to the number of cores on our computer and also tried to parallelize GLNS, then PCG's performance degraded due to running more total threads than the number of cores. While there is potential for combining PCG with PGLNS, e.g. using 2 PCG processes with 5 PGLNS threads per process on a 10-core computer, we leave this combination to future work.

After finding a tour $\Gamma^j$, each process $j$ generates a list of points $P^j$, called process $j$'s \emph{informed list} (Lines \ref{algline:pcg_init_informed_list}-\ref{algline:pcg_child_proc_loop_end}). The $i$th element $P^j[i]$ is the point at which $\Gamma^j$ visits target $i$. We then update the informed sets, such that $\mathcal{S}_i^\dagger$ contains all points associated with target $i$ from all informed lists in the current iteration (Lines \ref{algline:pcg_update_informed_sets_loop}-\ref{algline:pcg_update_informed_sets}). Finally, we update $\Gamma^*$ to the best tour found by any child process. If time remains, PCG runs its child processes again with the new informed sets. Otherwise, PCG returns the trajectory associated with $\Gamma^*$ (Line \ref{algline:pcg_return_trj}).

In PCG, a child process cannot begin a new iteration of the loop on Line \ref{algline:pcg_tour_improv_loop} until all other processes have finished their previous loop iteration. We also experimented with an asynchronous version of PCG, where as soon as a child process finishes invoking \TourViaGTSP, it gets the latest informed lists from the other processes (even if some of the lists have not been updated), constructs its own informed set, and begins the next iteration of the loop on Line \ref{algline:pcg_tour_improv_loop}. This did not provide noticeable benefits, so we pursued the synchronous algorithm.

\SetKwFunction{PCG}{PCG}

\begin{algorithm}\label{alg:PCG}
\caption{\protect\PCG}

\SetKwFunction{UniformRandomSample}{UniformRandomSample}
\SetKwFunction{ConstructGTSPGraph}{ConstructGTSPGraph}
\SetKwFunction{TransformToGTSP}{TransformToGTSP}
\SetKwFunction{DepthFirstSearch}{DepthFirstSearch}
\SetKwFunction{BestChildSolution}{BestChildSolution}
\SetKwFunction{GetInterceptionPoint}{GetInterceptionPoint}

\SetKwFor{ParFor}{parallel for}{do}{end}
\SetKwProg{Fn}{Function}{:}{}
\SetKwComment{Comment}{// }{}

$\Gamma^* = $ \GenerateInitialTour()\;\label{algline:pcg_gen_init_tour}
\lIf{\upshape $\Gamma^*$ is NULL}{return NULL}
\For{\upshape $i \in \mathcal{I}$}{
    $\mathcal{S}_i^\dagger = \{\GetInterceptionPoint(i, \Gamma^*)\}$\;\label{algline:pcg_init_informed_sets_after_dfs}
}

\Comment{Tour improvement}
\While{\upshape Time limit has not been reached}{\label{algline:pcg_tour_improv_loop}
    \ParFor{\upshape $j \in \{1, 2, \dots, n_\text{proc}\}$}{\label{algline:pcg_child_proc_loop}
        \For{\upshape $i \in \{1, 2, \dots, \ntar\}$}{\label{algline:pcg_child_proc_loop_begin}
            $\mathcal{S}_i^j = \mathcal{S}_i^\dagger\cup\RandomSamples(i, n_\text{rand})$\;\label{algline:pcg_sampling}
        }
        $\Gamma^j = \TourViaGTSP(\{\mathcal{S}_i^j\}_{i \in \mathcal{I}}, \Gamma^*)$\;\label{algline:pcg_tour_via_gtsp}
        Initialize $P^j$ as list of length $\ntar$\;\label{algline:pcg_init_informed_list}
        \For{\upshape $i \in \{1, 2, \dots, \ntar\}$}{
            $P^j[i] = \GetInterceptionPoint(i, \Gamma^j)$\;\label{algline:pcg_child_proc_loop_end}
        }
    }
    \For{$i \in \mathcal{I}$}{\label{algline:pcg_update_informed_sets_loop}
        $\mathcal{S}_i^\dagger = \{P^1[i], P^2[i], \dots, P^{n_\text{proc}}[i]\}$\;\label{algline:pcg_update_informed_sets}
    }
    Set $\Gamma^*$ equal to $\Gamma^{j}$ with least cost\;\label{algline:pcg_update_incumbent}
}

return trajectory associated with $\Gamma^*$\;\label{algline:pcg_return_trj}
\end{algorithm}
\vspace{-0.7cm}

\section{Theoretical Analysis}\label{sec:theoretical_analysis}
In this section, we show that under appropriate assumptions, all algorithms in the IRG framework are probabilistically complete: that is, if a feasible MT-TSP trajectory exists, the probability of finding a feasible trajectory approaches 1 as runtime tends to infinity. We additionally prove the asymptotic optimality of the IRG framework: that is, as runtime tends to infinity, the returned trajectory's cost converges to the optimum with probability 1. We specifically prove these properties for IRG-PGLNS, and PCG straightforwardly inherits them.

\begin{assumption}\label{assumption:complete_trj_gen}
For any $\qagt, t, \qagt', t' \in \Qagt \times \mathbb{R}^+ \times \Qagt \times \mathbb{R}^+$, if \TrajExists($\qagt, t, \qagt', t'$), then \GetTraj($\qagt, t, \qagt', t'$) returns a trajectory rather than NULL.
\end{assumption}

Note that Assumption \ref{assumption:complete_trj_gen} does not require a trajectory to exist between every pair of samples. Assumption \ref{assumption:complete_trj_gen} trivially holds for the Close-Enough MT-TSP and Robot Arm MT-TSP, since \GetTraj returns a straight-line trajectory from $(\qagt, t)$ to $(\qagt', t')$. For the Variable-Speed Dubins MT-TSP, while~\cite{chen2023elongation} provides a necessary and sufficient condition for the existence of a trajectory from $(\qagt, t)$ to $(\qagt', t')$ for a given speed $v$, it does not offer a constructive method for generating such a trajectory when one exists. Nevertheless, in our experiments, we successfully used numerical methods to compute a feasible trajectory for every tour returned by each algorithm in Section~\ref{sec:numerical_results}. Computing a trajectory for a tour never took more than 3 ms, which is much smaller than the 1 min time budget we allotted for computing tours themselves.

Let $\mathcal{F}$ be the set of feasible trajectories for an MT-TSP instance. For the Variable-Speed Dubins MT-TSP, we also constrain $\mathcal{F}$ to satisfy the requirements in Section \ref{sec:variable_speed_dubins_problem_setup}, i.e. each $\tau_\text{a} \in \mathcal{F}$ maintains constant speed between consecutive interceptions of two targets and only travels at speeds in $\mathcal{S}_\text{speed}$.

Let $\mathcal{L}_{\Qagt} = \prod\limits_{i = 1}^{\ntar} \Qagt \times [\underline{t}_i, \overline{t}_i]$, i.e. the set of lists of (configuration, time) pairs, with one pair per target. For a list $l_{\Qagt} \in \mathcal{L}_{\Qagt}$, if there exists a trajectory $\agenttrj \in \mathcal{F}$ intercepting each target $i$ at $l_{\Qagt}[i]$, we say $l_{\Qagt}$ is \emph{achieved by} $\agenttrj$.

\SetKwFunction{MapList}{MapList}
\SetKwFunction{ListCost}{ListCost}

Next, note that for each MT-TSP variant we consider, $\RandConfig_i(t)$ samples a parameter from some underlying sample space, then maps this parameter to an agent configuration $\qagt$. For example, as described in Section \ref{sec:problem_setup}, in the Close-Enough and Variable-Speed Dubins variants, the underlying space is $\mathbb{S}^1$, and in the robot arm variant, the underlying space is the space of solver-specific parameters for analytical IK and $\Qagt$ for numerical IK. In general, let $\mathcal{H}$ be the underlying sample space for the variant at hand, and let $\MapSample_i: \mathcal{H} \times [\underline{t}_i, \overline{t}_i] \rightarrow \Qagt$ be the function used by $\RandConfig_i$ to map a sample in $\mathcal{H}$ and a time in $[\underline{t}_i, \overline{t}_i]$ to an agent configuration in $\Qagt$. $\MapSample_i$ is defined as follows for the three variants.
\begin{itemize}
    \item Close-Enough MT-TSP: $\MapSample_i(h_i, t) = \tau_i(t) + r_i\begin{bmatrix}\cos(h_i)\\\sin(h_i)\end{bmatrix}$.
    
    \item Variable-Speed Dubins MT-TSP: $\MapSample_i(h_i, t) = (\tau_i(t), h_i)$.

    \item Robot Arm MT-TSP: $\MapSample_i(h_i, t)$ invokes the appropriate IK solver.
\end{itemize}

Let $\mathcal{L}_\mathcal{H} = \prod\limits_{i = 1}^{\ntar}(\mathcal{H} \times [\underline{t}_i, \overline{t}_i])$. Define a function $\MapList: \mathcal{L}_\mathcal{H} \rightarrow \mathcal{L}_{\Qagt}$ which maps a list $l_\mathcal{H}$ defined by $l_\mathcal{H}[i] = (h_i, t_i)$ to a list $l_{\Qagt}$ with $l_{\Qagt}[i] = (\MapSample_i(h_i, t_i), t_i)$. For $l_\mathcal{H} \in \mathcal{L}_\mathcal{H}$, if $\MapList(l_\mathcal{H})$ is achieved by some $\agenttrj \in \mathcal{F}$, we say $l_\mathcal{H}$ is achieved by $\agenttrj$ as well.

$\mathcal{L}_\mathcal{H}$ is a product of several sets, which we call components. Each component may have finite or infinite cardinality\footnote{Finite cardinality specifically occurs in the Robot Arm MT-TSP when using an analytical IK solver, since some solver-specific parameters are drawn from a finite set, e.g. the ``global configuration parameter" in \cite{faria2018position}.}. We define each component as a metric space by endowing each finite component with the discrete metric and each infinite component with the Euclidean metric \cite{lee2012smooth}. We endow $\mathcal{L}_{\mathcal{H}}$ with the metric defined by summing the metrics of its components. We use this metric when defining Lipschitz continuity of functions with domain $\mathcal{L}_\mathcal{H}$. We endow all other spaces with the Euclidean metric when necessary. We also endow each metric space with its metric topology, defining an open subset as a union of open balls under the space's metric \cite{lee2012smooth}. For any metric space $\mathcal{X}$ and $p, p' \in \mathcal{X}$, let $d(p, p')$ be the distance from $p$ to $p'$ under the metric. For any $p \in \mathcal{X}$ and $\mathcal{S} \subseteq \mathcal{X}$, let $d(p, \mathcal{S}) = \inf_{p' \in \mathcal{S}}d(p, p')$.


\begin{theorem}\label{thm:prob_complete}
(Probabilistic Completeness) Suppose there exists a non-empty open set $\mathcal{N} \subseteq \mathcal{L}_\mathcal{H}$ such that each $l_\mathcal{H} \in \mathcal{N}$ is achieved by some $\agenttrj \in \mathcal{F}$. Given Assumption \ref{assumption:complete_trj_gen}, as the number of iterations of Alg. \ref{alg:InitialTourGen} tends to infinity, the probability that Alg. \ref{alg:InitialTourGen}, and therefore IRG-PGLNS, produces a feasible trajectory approaches 1.
\end{theorem}

\begin{proof}
At the $k$th iteration of the loop in Alg. \ref{alg:InitialTourGen} Line \ref{algline:init_tour_main_loop}, let $({}^kh_i, {}^kt_i)$ be the final random sample drawn for target $i$, and define ${}^kl_\mathcal{H} \in \mathcal{L}_\mathcal{H}$ by ${}^kl_\mathcal{H}[i] = ({}^kh_i, {}^kt_i)$. We sampled each $({}^kh_i, {}^kt_i)$ uniformly from $\mathcal{H} \times [\underline{t}_i, \overline{t}_i]$, so ${}^kl_\mathcal{H}$ is sampled from a uniform distribution over $\mathcal{L}_\mathcal{H}$, which places positive measure on nonempty open subsets, including $\mathcal{N}$. Thus as $k$ approaches infinity, the probability that ${}^kl_\mathcal{H} \in \mathcal{N}$ approaches 1. Once we have ${}^kl_\mathcal{H} \in \mathcal{N}$, the tour obtained by sorting $\MapList({}^kl_\mathcal{H})$ in order of increasing time values, then prepending $(\qagtzero, 0)$, will be feasible for the GTSP. Alg. \ref{alg:InitialDFS} will return this tour or some other feasible tour, ensuring that Alg. \ref{alg:IRG_PGLNS} produces a feasible tour $\Gamma^*$. Assumption \ref{assumption:complete_trj_gen} then guarantees that we connect the consecutive points in $\Gamma^*$ to return a feasible trajectory.
\end{proof}


Now we begin the asymptotic optimality proofs. First, we define the \emph{corresponding} target sequence of $l_\mathcal{H} \in \mathcal{L}_\mathcal{H}$ as the sequence obtained by sorting the interception times in $l_\mathcal{H}$. For a set $\mathcal{S}$ in a topological space, let $\overline{\mathcal{S}}$ be the closure of $\mathcal{S}$.

Let $l_\mathcal{H}^* \in \mathcal{L}_\mathcal{H}$ be a list that is achieved by an optimal $\tau_\textnormal{a}^* \in \mathcal{F}$.

\begin{assumption}\label{assumption:N_existence}
A nonempty open $\mathcal{N} \subseteq \mathcal{L}_\mathcal{H}$ exists such that (i) $l_\mathcal{H}^* \in \overline{\mathcal{N}}$, (ii) all lists in $\overline{\mathcal{N}}$ have the same corresponding target sequence, and (iii) each $l_\mathcal{H} \in \overline{\mathcal{N}}$ is achieved by some $\tau_\text{a} \in \mathcal{F}$.
\end{assumption}
We validate Assumption 2 for simple instances where the optimum can be found in Section \ref{sec:mt_tsp_eval_optimality}.

\begin{assumption}\label{assumption:lipschitz_target_trj}
The trajectories of the targets are Lipschitz on their time windows.
\end{assumption}

\begin{assumption}\label{assumption:robot_arm_assumptions}
For the robot-arm MT-TSP, suppose the arm is the KUKA iiwa, and that we are using analytical IK. Suppose that for all configurations in $\MapList(l_\mathcal{H}^*$), (i) the shoulder-wrist vector (defined in \cite{faria2018position}) has nonzero norm, and (ii) no joint angle equals 0 or $\pi$.
\end{assumption}

\begin{lemma}\label{lemma:maplist_lipschitz}
{Suppose Assumptions \ref{assumption:lipschitz_target_trj} and \ref{assumption:robot_arm_assumptions} hold.} For the Close-Enough MT-TSP and Variable-Speed Dubins MT-TSP, \MapList is Lipschitz on $\mathcal{L}_\mathcal{H}$. For the Robot Arm MT-TSP, there is an open $\mathcal{N}^* \subseteq \mathcal{L}_\mathcal{H}$ containing $l_\mathcal{H}^*$ such that \MapList is Lipschitz on $\overline{\mathcal{N}^*}$.
\end{lemma}
\begin{proof}
Since \MapList concatenates invocations of $\MapSample_i$, showing that \MapList is Lipschitz on $\mathcal{L}_\mathcal{H}$ requires showing that $\MapSample_i$ is Lipschitz on the $i$th component of $\mathcal{L}_\mathcal{H}$, i.e. $\mathcal{H} \times [\underline{t}_i, \overline{t}_i]$, for each $i \in \mathcal{I}$. For the Close-Enough MT-TSP, $\MapSample_i$ is Lipschitz on $\mathcal{H} \times [\underline{t}_i, \overline{t}_i]$, since $\MapSample_i(\theta, t)$ only consists of cosine, sine, and addition, all of which are Lipschitz, as well as $\tau_i$, which is Lipschitz by Assumption \ref{assumption:lipschitz_target_trj}. For the Variable-Speed Dubins MT-TSP, $\MapSample_i$ is Lipschitz on $\mathcal{H} \times [\underline{t}_i, \overline{t}_i]$, since $\MapSample_i(\phi, t)$ concatenates $\tau_i(t)$ and $\phi$, which is a linear operation, and $\tau_i$ is Lipschitz by Assumption \ref{assumption:lipschitz_target_trj}. Thus \MapList is Lipschitz on $\mathcal{L}_\mathcal{H}$ for the Close-Enough MT-TSP and Variable-Speed Dubins MT-TSP.

\SetKwFunction{arctantwo}{atan2}
\SetKwFunction{arccos}{acos}

Now consider the Robot Arm MT-TSP for the iiwa. We proceed by constructing an open set $\mathcal{P}_i$ for each $i \in \mathcal{I}$ where $\MapSample_i$ is Lipschitz on $\overline{\mathcal{P}_i}$, then letting $\mathcal{N}^* = \prod\limits_{i = 1}^\ntar \mathcal{P}_i$. Then \MapList will be Lipschitz on $\prod\limits_{i = 1}^\ntar\overline{\mathcal{P}}_i$, and since $\overline{\mathcal{N}^*} = \overline{\prod\limits_{i = 1}^\ntar\mathcal{P}_i} \subseteq \prod\limits_{i = 1}^\ntar\overline{\mathcal{P}}_i$, \MapList will be Lipschitz on $\overline{\mathcal{N}^*}$ as well.

Computing $\MapSample_i(h_i, \tau_i(t_i))$ via analytical IK requires computing a shoulder-wrist vector $w_i = f_\text{sw}(\tau_i(t_i))$, then normalizing $w_i$ (see \cite{faria2018position}). Let $f_i: [\underline{t}_i, \overline{t}_i] \rightarrow \mathbb{R}^3$ where $f_i(t_i) = f_\text{sw}(\tau_i(t_i))$. Let $\hat{f}_i(t_i) = \frac{f_i(t_i)}{\|f_i(t_i)\|}$. For each joint $j \in \{1, 3, 5, 7\}$, $\MapSample_i(h_i, t_i)$ computes the $j$th joint angle $q^j_i$ as $\arctantwo(g^j(\tau_i(t_i), h_i, \hat{f}_i(t_i)))$, where $g^j$ is defined for each $j$ in \cite{faria2018position}, returning a $(y, x)$ pair. For joints $j \in \{2, 4, 6\}$, $\MapSample_i(h_i, t_i)$ computes $q_i^j$ as $\arccos(g^j(\tau_i(t_i), h_i, \hat{f}_i(t_i)))$, where $g^j$ is again defined for each $j$ in \cite{faria2018position}, but scalar-valued.

Let $(h_i^*, t_i^*) = l_\mathcal{H}^*[i]$. Let $w_i^* = f_\text{sw}(\tau_i(t_i^*))$. Let $\overline{\mathcal{B}_i} \in \mathbb{R}^3$ be the closed ball centered at $(0, 0, 0)$ with radius $\|w_i^*\|/2$, and let $\mathcal{C}_i = \mathbb{R}^3 \setminus \overline{\mathcal{B}_i}$. Let $f^{-1}_i(\mathcal{C}_i)$ be the preimage of $\mathcal{C}_i$ under $f_i$. Let $g_i: \mathcal{H} \times \overline{f^{-1}_i(\mathcal{C}_i)} \rightarrow \mathbb{R}^{11}$ be defined by
\begin{align}
g_i(h_i, t_i) = (&g^1(\tau_i(t_i), h_i, \hat{f}_i(t_i)), g^2(\tau_i(t_i), h_i, \hat{f}_i(t_i)),\\
\dots, &g^7(\tau_i(t_i), h_i, \hat{f}_i(t_i)))\nonumber
\end{align}

For each $j$, let ${g_i^j}^* = g^j(\tau_i(t_i^*), h_i^*, \hat{f}_i(t_i^*))$. For each $j \in \{1, 3, 5, 7\}$, ${g_i^j}^* \notin \{0\} \times \mathbb{R}$ by Assumption \ref{assumption:robot_arm_assumptions} (ii), and $\{0\} \times \mathbb{R}$ is closed, so $d({g_i^j}^*, \{0\} \times \mathbb{R})$ is some $\delta_i^j > 0$. For each $j \in \{2, 4, 6\}$. ${g_i^j}^* \notin \{-1, 1\}$ by Assumption \ref{assumption:robot_arm_assumptions} (ii), and $\{-1, 1\}$ is closed, so $d({g_i^j}^*, \{-1, 1\})$ is some $\delta_i^j > 0$. For each $j$, let ${g_i^j}^* = g^j(\tau_i(t_i^*), h_i^*, \hat{f}_i(t_i^*))$, and let $\mathcal{B}_i^j$ be the open ball centered at ${g_i^j}^*$ with radius $\delta_i^j/2$. Let $\mathcal{P}_i = {g}_i^{-1}\left(\prod\limits_{j = 1}^7 \mathcal{B}_i^j\right) \cap (\mathcal{H} \times f^{-1}_i(\mathcal{C}_i))$. We now show that $\MapSample_i$ is Lipschitz on $\overline{\mathcal{P}}_i$.

First, normalization is Lipschitz on $\overline{\mathcal{C}}_i$, since each element's norm is lower-bounded by $\|w_i^*\|/2$, and $\|w_i^*\|/2 > 0$ by Assumption \ref{assumption:robot_arm_assumptions} (i). Also, $f_\text{sw}$ is Lipschitz \cite{faria2018position}, and $\tau_i$ is Lipschitz by Assumption \ref{assumption:lipschitz_target_trj}. Thus $\hat{f}_i$ is Lipschitz on $f_i^{-1}(\overline{\mathcal{C}}_i)$, and thereby Lipschitz on $\overline{f_i^{-1}(\mathcal{C}_i)} \subseteq f_i^{-1}(\overline{\mathcal{C}}_i)$. Next, each $g^j$ is Lipschitz, so $g_i$, which composes $g^j$ with $\tau_i$ and $\hat{f}_i$, is Lipschitz on $\mathcal{H} \times \overline{f_i^{-1}(\mathcal{C}_i)}$.

For a set $\mathcal{S}$, let $g_i(\mathcal{S})$ be the image of $\mathcal{S}$ under $g_i$. We have
\begin{align}
    g_i(\overline{\mathcal{P}_i}) \subseteq \overline{g_i(\mathcal{P}_i)} \subseteq \overline{\prod\limits_{j = 1}^7 \mathcal{B}_i^j} \subseteq \prod\limits_{j = 1}^7 \overline{\mathcal{B}_i^j}.\label{eqn:subset_of_product_of_closures}
\end{align}
The first containment in \eqref{eqn:subset_of_product_of_closures} follows from the continuity of $g_i$, and the second containment follows from $g_i(\mathcal{P}_i) \subseteq \prod\limits_{j = 1}^7 \mathcal{B}_i^j$. The third containment is a property of Cartesian products.

For joints $j \in \{1, 3, 5, 7\}$, \arctantwo is Lipschitz on $\overline{\mathcal{B}_i^j}$, and thus Lipschitz on the 1st, 3rd, 5th, and 7th components of the LHS of \eqref{eqn:subset_of_product_of_closures}. For joints $j \in \{2, 4, 6\}$, \arccos is Lipschitz on $\overline{\mathcal{B}_i^j}$, so \arccos is Lipschitz on the 2nd, 4th, and 6th components of the LHS of \eqref{eqn:subset_of_product_of_closures}. These facts, and the Lipschitzness of $g_i$ on its domain (and thereby $\overline{\mathcal{P}}_i$), imply that $\MapSample_i$, which composes $\arctantwo$ and $\arccos$ with $g_i$, is Lipschitz on $\overline{\mathcal{P}}_i$.

Next, each $\mathcal{P}_i$ is open, since $g_i$ and $f_i$ are continuous, $\mathcal{B}_i^j$ and $\mathcal{C}_i$ are open, and preimages of open sets under continuous functions are open. Thus $\mathcal{N}^*$ is open. Finally, $\mathcal{N}^*$ contains $l_\mathcal{H}^*$, so $\mathcal{N}^*$ is nonempty.
\end{proof}

Next, define the function $\ListCost: \mathcal{L}_\mathcal{H} \rightarrow \mathbb{R}$, which maps a list $l_\mathcal{H}$ to the cost of its associated trajectory, if one exists, and $\infty$ otherwise.

\begin{assumption}\label{assumption:vs_dubins_assumption}
For the Variable-Speed Dubins MT-TSP, for a list $l_\mathcal{H}$, let $(\qagt_{k - 1}, t_{k - 1})$, $(\qagt_k, t_k)$ be the $k$th consecutive pair of points in the associated tour. Let $v_k^*(l_\mathcal{H}) = v^*(\qagt_{k - 1}, t_{k - 1}, \qagt_k, t_k)$, where $v^*$ is defined in Section \ref{subsec:tour_via_gtsp}. Suppose all $l_\mathcal{H} \in \overline{\mathcal{N}}$ are achieved by trajectories using the same sequence of speeds, i.e. for all $l_\mathcal{H}, l_\mathcal{H}' \in \overline{\mathcal{N}}$, for all $k$, $v_k^*(l_\mathcal{H}) = v_k^*(l_\mathcal{H}')$.
\end{assumption}

\begin{lemma}\label{lemma:lipschitz} Suppose Assumptions \ref{assumption:lipschitz_target_trj}-\ref{assumption:vs_dubins_assumption} hold, and for the Robot Arm MT-TSP, redefine $\mathcal{N}$ as $\mathcal{N} \cap \mathcal{N}^*$. Then \ListCost is Lipschitz on $\overline{\mathcal{N}}$.
\end{lemma}

\begin{proof}
Let $l_\mathcal{H} = ((h_1, t_1), (h_2, t_2), \dots, (h_\ntar, t_\ntar)) \in \mathcal{N}$. First, consider the Close-Enough and Robot Arm MT-TSPs. \ListCost performs the following operations in order:
\begin{enumerate}
    \item Map $l_\mathcal{H}$ to a list $l_{\Qagt} \in \mathcal{L}_\Qagt$ via \MapList\label{enumitem:list_cost_step1}
    \item Map $l_\Qagt$ to a tour $\Gamma$ by sorting in order of time and prepending $(\qagtzero, 0)$\label{enumitem:list_cost_step2}
    \item Compute the differences between consecutive configurations in $\Gamma$ and sum the norms of the differences.\label{enumitem:list_cost_step3}
\end{enumerate}
Step \ref{enumitem:list_cost_step1} is Lipschitz by Lemma \ref{lemma:maplist_lipschitz}. Since all lists in $\overline{\mathcal{N}}$ have the same corresponding target sequence by Assumption \ref{assumption:N_existence}, the sort in Step \ref{enumitem:list_cost_step2} is a fixed permutation, which is Lipschitz. The prepend in Step \ref{enumitem:list_cost_step2} and all operations in \ref{enumitem:list_cost_step3} are also Lipschitz. Since \ListCost is a composition of Lipschitz functions, \ListCost is Lipschitz on $\overline{\mathcal{N}}$.

Now consider the Variable-Speed Dubins MT-TSP. \ListCost performs steps 1) to 3) above, except in 3), we now sum the values of $v_k^*(l_\mathcal{H})(t_k - t_{k - 1})$ for consecutive times $t_{k - 1}, t_k$ in $\Gamma$. Since $v_k^*(l_\mathcal{H})$ is constant on $\overline{\mathcal{N}}$ by Assumption \ref{assumption:vs_dubins_assumption}, this summation is linear and thereby Lipschitz in $\Gamma$. Since \ListCost is a composition of Lipschitz functions, \ListCost is Lipschitz on $\overline{\mathcal{N}}$.
\end{proof}

\begin{theorem}\label{thm:asymptotic-optimality}
(Asymptotic optimality) Suppose Assumptions \ref{assumption:complete_trj_gen}-\ref{assumption:vs_dubins_assumption} hold. Let $c^* = \ListCost(l_\mathcal{H}^*)$. As the number of iterations of the loop in Alg. \ref{alg:IRG_PGLNS} Line \ref{algline:irg_tour_improvement_loop} tends to infinity, the cost of the returned trajectory converges to $c^*$ with probability 1.
\end{theorem}

\begin{proof}
Let $\epsilon > 0$ be fixed and arbitrary, and let $M$ be the Lipschitz constant of \ListCost. For the remainder of the proof, for the Robot Arm MT-TSP, redefine $\mathcal{N}$ as $\mathcal{N} \cap \mathcal{N}^*$, as we did in Lemma \ref{lemma:lipschitz}. Let $\mathcal{B}_{\epsilon/M}(l_\mathcal{H}^*) = \{l_\mathcal{H} \in \mathcal{L}_\mathcal{H} \;|\; d(l_\mathcal{H}, l_\mathcal{H}^*) < \epsilon/M\}$, and let $\mathcal{D}_{\epsilon/M} = \mathcal{B}_{\epsilon/M}(l_\mathcal{H}^*) \cap \mathcal{N}$. Since $l_\mathcal{H}^* \in \overline{\mathcal{N}}$, $l_\mathcal{H}^* \in \mathcal{N}$ or $l_\mathcal{H}^*$ is a limit point of $\mathcal{N}$. If $l_\mathcal{H}^* \in \mathcal{N}$, then $\mathcal{D}_{\epsilon/M}$ is nonempty. If $l_\mathcal{H}^*$ is a limit point of $\mathcal{N}$, then every open set containing $l_\mathcal{H}^*$, in particular, $\mathcal{B}_{\epsilon/M}$, contains some point in $\mathcal{N}$ distinct from $l_\mathcal{H}^*$, so $\mathcal{D}_{\epsilon/M}$ is nonempty. Also, $\mathcal{D}_{\epsilon/M}$ is open since it is an intersection of open sets. In the following text, let $\mathbb{P}[E]$ denote the probability of an event $E$.

Consider an arbitrary loop iteration $k$. Let $({}^kh_i, {}^kt_i)$ be the final random sample drawn for target $i$, and define ${}^kl_\mathcal{H} \in \mathcal{L}_\mathcal{H}$ by ${}^kl_\mathcal{H}[i] = ({}^kh_i, {}^kt_i)$. $\mathcal{D}_{\epsilon/M}$ has positive measure under our sampling distribution by the same argument as Theorem \ref{thm:prob_complete}, which means that for some ${}^\epsilon p > 0$, $\mathbb{P}[{}^kl_\mathcal{H} \in \mathcal{D}_{\epsilon/M}] = {}^\epsilon p$. If ${}^kl_\mathcal{H} \in \mathcal{D}_{\epsilon/M}$, we have the following:
\begin{align}
    |\ListCost({}^kl_\mathcal{H}) - c^*| &\leq Md({}^kl_\mathcal{H}, l_\mathcal{H}^*) < \epsilon\label{eqn:apply_lipschitz_listcost_and_ball}\\
    \ListCost({}^kl_\mathcal{H}) - c^* &< \epsilon\label{eqn:remove_abs_val}\\
    \ListCost({}^kl_\mathcal{H}) &< c^* + \epsilon\label{eqn:add_cstar}\\
    \mathbb{P}[\ListCost({}^kl_\mathcal{H}) < c^* + \epsilon] &\geq \mathbb{P}[{}^kl_\mathcal{H} \in \mathcal{D}_{\epsilon/M}] = {}^\epsilon p\label{eqn:prob_less_cstar_plus_eps}
\end{align}
The first inequality in \eqref{eqn:apply_lipschitz_listcost_and_ball} follows from \ListCost being Lipschitz on $\overline{\mathcal{N}}$ with Lipschitz constant $M$, and the second inequality follows from ${}^kl_\mathcal{H} \in \mathcal{B}_{\epsilon/M}(l_\mathcal{H}^*)$. \eqref{eqn:remove_abs_val} follows from \eqref{eqn:apply_lipschitz_listcost_and_ball} and the fact that $\ListCost({}^kl_\mathcal{H}) - c^* \geq 0$ by the optimality of $c^*$. \eqref{eqn:add_cstar} adds $c^*$ to both sides of \eqref{eqn:remove_abs_val}. \eqref{eqn:prob_less_cstar_plus_eps} follows from the fact that ${}^kl_\mathcal{H} \in \mathcal{D}_{\epsilon/M}$ implies \eqref{eqn:add_cstar}.

For any iteration $l$, let ${}^lc$ be the cost of the tour produced, let ${}^lc^*$ be the cost of the optimal tour on the current sample point graph, and let ${}^lZ$ be the event that for all $m \leq l$, ${}^mc > c^* + \epsilon$. \eqref{eqn:prob_less_cstar_plus_eps}, along with the independence of ${}^kl_\mathcal{H}$ from ${}^{k - 1}Z$, implies
\begin{align}
    \mathbb{P}[\ListCost({}^kl_\mathcal{H}) < c^* + \epsilon|{}^{k - 1}Z] &\geq {}^\epsilon p\label{eqn:prob_less_cstar_plus_eps_given_Zk}\\
    \mathbb{P}[{}^kc^* < c^* + \epsilon|{}^{k - 1}Z] &\geq {}^\epsilon p.\label{eqn:prob_opt_cost_less_cstar_plus_eps}
\end{align}
Also, since \TourViaGTSP attempts to randomly sample an optimal tour after running PGLNS,
\begin{align}
    \mathbb{P}[{}^kc = {}^kc^*|{}^{k - 1}Z] \geq \frac{1}{\ntar !}\left(\frac{1}{n_\text{rand} + 1}\right)^{n_\text{tar}}.\label{eqn:prob_return_opt_cost}
\end{align}
In \eqref{eqn:prob_return_opt_cost}, $\frac{1}{\ntar !}$ is the probability of selecting the optimal sequence of targets using uniform random sampling, and $\left(\frac{1}{n_\text{rand} + 1}\right)^{n_\text{tar}}$ is the probability of randomly selecting the optimal node in each target's cluster using uniform random sampling. Denote the RHS of \eqref{eqn:prob_return_opt_cost} as $\underline{p}$. \eqref{eqn:prob_opt_cost_less_cstar_plus_eps} and \eqref{eqn:prob_return_opt_cost} imply \mbox{$\mathbb{P}[{}^kc < c^* + \epsilon|{}^{k - 1}Z] \geq {}^\epsilon p \underline{p} > 0$}. This means
\begin{align}
    \mathbb{P}[{}^kc > c^* + \epsilon|{}^{k - 1}Z] \leq 1 - {}^\epsilon p \underline{p}.\label{eqn:prob_return_greater_cstar_plus_eps}
\end{align}
Denote the RHS of \eqref{eqn:prob_return_greater_cstar_plus_eps} as $\zeta$. Since \mbox{${}^kZ = ({}^kc > c^* + \epsilon) \cap ({}^{k - 1}Z)$}, Bayes' rule tells us that
\begin{align}
    \mathbb{P}[{}^kZ] &= \mathbb{P}[{}^kc > c^* + \epsilon|{}^{k - 1}Z]\mathbb{P}[{}^{k - 1}Z] \leq \zeta\mathbb{P}[{}^{k - 1}Z]\label{eqn:Zk}
\end{align}
where the second inequality follows from \eqref{eqn:prob_return_greater_cstar_plus_eps}.
\eqref{eqn:Zk} can then be used within an induction to show that $\mathbb{P}[{}^kZ] \leq \zeta^k$ for all $k$. We omit the induction here for brevity.

Note that ${}^kZ \iff {}^kc > c^* + \epsilon$. This is because in IRG-PGLNS, tour cost never increases from one iteration to the next, so ${}^kc > c^* + \epsilon$ is only possible if ${}^mc > c^* + \epsilon$ for all $m \leq k$. Thus, $\mathbb{P}[{}^kc > c^* + \epsilon] \leq \zeta^k$. Combining this with $\zeta < 1$ (implied by ${}^\epsilon p \underline{p} > 0$), we have \mbox{$\sum\limits_{k = 1}^\infty\mathbb{P}[{}^kc > c^* + \epsilon] < \infty$}. Equivalently, \mbox{$\sum\limits_{k = 1}^\infty\mathbb{P}[{}^kc - c^* > \epsilon] < \infty$}. The optimality of $c^*$ implies ${}^kc - c^* \geq 0$, so $\sum\limits_{k = 1}^\infty\mathbb{P}[|{}^kc - c^*| > \epsilon] < \infty$. The Borel-Cantelli Lemma \cite{grimmett2020probability} then implies $\mathbb{P}[\limsup\limits_{k \rightarrow \infty}[|{}^kc - c^*| > \epsilon]] = 0$, i.e. ${}^kc$ converges to $c^*$ almost surely. By Assumption \ref{assumption:complete_trj_gen}, the returned trajectory's cost converges to $c^*$ almost surely.
\end{proof}

Note that the statement $\mathbb{P}[{}^kZ] \leq \zeta^k$ in the proof of Theorem \ref{thm:asymptotic-optimality} is a finite-time guarantee on IRG-PGLNS's performance. That is, after $k$ iterations, $\zeta^k$ upper-bounds the probability that the incumbent's cost is worse than $c^* + \epsilon$. This finite-time guarantee also lower-bounds the convergence rate of IRG-PGLNS, where we use the term convergence rate in the sense of Definition 1.4 in \cite{suli2003introduction}. In particular, $\mathbb{P}[{}^k Z] \leq \zeta^k$ implies that $\mathbb{P}[{}^k Z]$ converges to zero with rate at least $-\log_{10}(\zeta)$.

In practice, however, $\zeta^k$ is a loose upper bound on $\mathbb{P}[{}^k Z]$. We considered several instances where the optimal cost $c^*$ is known (i.e. the 10 and 20 target instances from Section \ref{sec:mt_tsp_eval_optimality}), set $\epsilon = {}^kc - c^*$, where ${}^kc$ is the cost upon termination of IRG-PGLNS, and computed $\zeta^k$. We found that $\zeta^k$ was essentially 1 in all cases. Finding a stronger finite-time guarantee is a direction for future work.

\section{Complexity Analysis}\label{sec:complexity_analysis}
In this section, we analyze the space complexity of the IRG algorithms.\footnote{We focus on space complexity rather than time complexity because the algorithms are anytime, so time complexity is simply determined by the user-specified time limit.} First, we examine the initial tour generation (Alg. \ref{alg:InitialTourGen}), common to IRG-PGLNS and PCG. The worst-case space complexity is unbounded due to the unbounded number of points Alg. \ref{alg:InitialTourGen} may add to the initial set of $n_\text{rand,init}$ points per target. However, we show in Section \ref{sec:numerical_results} that for all instances tested, $n_\text{rand,init}$ can be set large enough that points never need to be added. Thus, we analyze space complexity assuming only $n_\text{rand,init}$ points are used per target.

The operations in Alg. \ref{alg:InitialTourGen} are random sampling, computing a matrix of edge costs between pairs of sampled points, and solving a GTSP with the DFS in Alg. \ref{alg:InitialDFS}. The random samples require $O(\ntar n_\text{rand,init})$ space, and the cost matrix requires $O(\ntar^2 n_\text{rand,init}^2)$ space. Each search node in the DFS requires $O(\ntar)$ space to store its visited subset. The stack and closed lists could each store every possible search node in the worst case. The number of search nodes is in $O(2^\ntar n_\text{rand,init})$, where $2^\ntar$ is the number of possible visited subsets, and $n_\text{rand,init}$ is the number of possible terminal points, apart from the initially generated search node, whose only terminal point is $(\qagtzero, 0)$. This implies space complexity $O(2^\ntar n_\text{rand,init} \ntar)$ for the DFS, and thus a worst-case space complexity $O(2^\ntar n_\text{rand,init} \ntar)$ for Alg. \ref{alg:InitialTourGen}, assuming only $n_\text{rand,init}$ points are used.

Now we examine the space complexity of the tour improvement, starting with IRG-PGLNS. Each tour improvement iteration requires sampling and computing edge costs, which have space complexities $O(\ntar n_\text{rand})$ and $O(\ntar^2 n_\text{rand}^2)$, respectively. Each iteration then solves a GTSP with PGLNS. Within PGLNS, we refer to each iteration of the loop on Alg. \ref{alg:PGLNS} Line \ref{algline:pglns_thread_inner_loop} as a remove/insert iteration. In a remove/insert iteration, the current thread makes a copy of $\Gamma^c$, requiring $O(n_\text{thread}\ntar)$ space after we sum over threads. When a single thread performs a remove/insert iteration with $n_\text{rem}$ nodes removed, it uses additional space at most $O(n_\text{rem})$ \cite{smith2017glns}. We use PGLNS in fast mode, where $n_\text{rem} \in O(1)$. Thus the space complexity of an iteration of tour improvement is $O(\ntar^2 n_\text{rand}^2 + n_\text{thread}\ntar)$.

For an iteration of tour improvement in PCG, we have the single-thread space complexity of an iteration of IRG-PGLNS, multiplied by $n_\text{proc}$. This yields space complexity $O(n_\text{proc}\ntar^2 n_\text{rand}^2)$.

\section{Numerical Results}\label{sec:numerical_results}
We ran experiments on an Intel i9-9820X 3.3GHz CPU with 128 GB RAM and 10 cores. Section \ref{sec:generating_instances} describes how we generate problem instances, and Section \ref{sec:tuning} describes how we tune the planners under comparison. In Sections \ref{sec:vary_num_targets} to \ref{sec:ramt_tsp_vary_vmax}, we compare IRG-PGLNS and PCG to a baseline and two ablations. Our baseline for the Close-Enough MT-TSP and Variable-Speed Dubins MT-TSP is a modification of the memetic algorithm \cite{ding2022memetic}, which addresses the fixed-speed Dubins Close-Enough MT-TSP without time windows. In particular, we extended \cite{ding2022memetic} to handle time windows, we specialized \cite{ding2022memetic} to the Close-Enough MT-TSP and Variable-Speed Dubins MT-TSP, and we parallelized \cite{ding2022memetic}, with details in Appendix \ref{appendix:memetic_modifications}. We did not run this baseline for the Robot Arm MT-TSP, since \cite{ding2022memetic} did not consider a robot arm. Our first ablation, IRG-GLNS, is the same as IRG-PGLNS except that IRG-GLNS uses GLNS instead of PGLNS as the GTSP solver. Our second ablation, called Parallel Decoupled GTSPs (PDG), is an ablation of PCG where the processes do not communicate. PDG runs Alg. \ref{alg:IRG_PGLNS} Line \ref{algline:irg_gen_init_tour}, then spawns several parallel child processes, each running Lines \ref{algline:irg_tour_improvement_loop}-\ref{algline:irg_end_tour_improvement_loop} and updating its own copy of $\Gamma^*$. Upon reaching the planning time budget, PDG returns the best $\Gamma^*$ over all child processes. Both IRG-GLNS and PDG fall into the IRG framework.

The optimal solutions for the problem instances in Sections \ref{sec:vary_num_targets} to \ref{sec:ramt_tsp_vary_vmax} are unknown, since there are no optimal solvers for the Close-Enough, Variable-Speed Dubins, and Robot Arm MT-TSPs. To evaluate how close IRG-PGLNS and PCG come to the optimum, in Section \ref{sec:mt_tsp_eval_optimality}, we generated a set of instances for what we call the \emph{Linear MT-TSP}, i.e. the MT-TSP from Fig. \ref{fig:intro_fig} (a), but targets move with constant velocity. Optimal solutions for the Linear MT-TSP can be found using \cite{philip2024mixedinteger}. In Section \ref{sec:mt_tsp_eval_optimality}, we show that IRG-PGLNS and PCG converge closely to the optimum on instances where we can find the optimum, and that they converge faster than \cite{philip2024mixedinteger} for larger numbers of targets. Finally, in Section \ref{sec:mt_tsp_eval_init_tour_gen}, we evaluate IRG's method of solving the GTSP on its first set of sample points, showing that our method finds feasible solutions more quickly than existing GTSP algorithms.


Each planner was given a 30 s time budget per instance for the Close-Enough MT-TSP and a 1 minute time budget per instance in the Variable-Speed Dubins and Robot Arm MT-TSPs. We used a shorter budget for the Close-Enough MT-TSP because sampling and edge cost computations are relatively inexpensive, and all planners tend to converge more quickly than in the other variants. When constructing $\mathcal{G}$, IRG-PGLNS and IRG-GLNS invoked \TrajExists for different node pairs in parallel using 10 threads. Also, for the Robot Arm MT-TSP, to mitigate the computational expense of IK during sampling, IRG-GLNS and IRG-PGLNS parallelized the loop on Line \ref{algline:irg_sampling_loop} of Alg. \ref{alg:IRG_PGLNS}. PCG and PDG, on the other hand, performed this loop serially, since they already sample several sets of points for each target in parallel. When generating an initial tour with Alg. \ref{alg:InitialTourGen}, we parallelized edge cost computations in $\mathcal{G}$ for all MT-TSP variants, parallelized sampling for the Robot Arm MT-TSP, and generated search nodes' successors in parallel during the DFS.

\subsection{Generating Problem Instances}\label{sec:generating_instances}

We first generated 20 instances for each problem variant, which we call \emph{default} instances, using \emph{default} parameter values specified for each variant in Sections \ref{sec:vsdmt_tsp_instance_gen}-\ref{sec:ramt_tsp_instance_gen}. The default number of targets is 200 for all variants. We then varied parameters within the default instances to generate additional instances. For each default instance, we sampled $\qagtzero$ uniformly at random from a set $\mathcal{Q}_{a,0}$, specified for each variant in Sections \ref{sec:vsdmt_tsp_instance_gen}-\ref{sec:ramt_tsp_instance_gen}. To generate trajectories of targets such that each instance was feasible, we extended the method from \cite{bhat2024AComplete}. We first generated a random feasible tour, $((q_{\textnormal{a}, 0}, 0), (q_{\textnormal{a}, 1}, t_1), \dots (q_{\textnormal{a}, n_\text{tar}}, t_{n_\text{tar}}))$, as well as a random, two-segment piecewise-linear trajectory $\tilde{\tau}_i$ for each $i \in \mathcal{I}$ that passed through $(q_{\textnormal{a}, i}, t_i)$. Then for each $i \in \mathcal{I}$, we fit a cubic B-spline that passed through $(q_{\textnormal{a}, i}, t_i)$ and the endpoints of the segments of $\tilde{\tau}_i$ to obtain the final trajectory $\tau_i$. We provide instance generation code in the multimedia attachment.

\subsubsection{Close-Enough MT-TSP}\label{sec:cemt_tsp_instance_gen}
For the Close-Enough MT-TSP, $\mathcal{Q}_{a,0} = [-50\text{ m}, 50 \text{ m}]^2$. The agent's speed limit was 5 m/s in all instances. All targets had time windows of length 108 s. Each segment of $\tilde{\tau}_i$ had a direction sampled uniformly at random and a speed uniformly random sampled from the range [0.5, 1] m/s. Within a single instance, all targets have a common radius value, which we call the \emph{target radius}. The default target radius was 12 m.

\subsubsection{Variable-Speed Dubins MT-TSP}\label{sec:vsdmt_tsp_instance_gen}
For the Variable-Speed Dubins MT-TSP, $\mathcal{Q}_{a,0} = [-50 \text{ m}, 50\text{ m}]^2 \times \mathbb{S}^1$. The agent's default $v_\text{min}$ was 2 m/s. In all instances, $v_\text{max} = 5$ m/s and $\omega_\text{max} = 0.25$ rad/s. We let $\mathcal{S}_\text{speed} = \{v_\text{min}, \frac{2}{3}v_\text{min} + \frac{1}{3}v_\text{max}, \frac{1}{3}v_\text{min} + \frac{2}{3}v_\text{max}, v_\text{max}\}$. All targets had time windows of length 108 s. Each segment of $\tilde{\tau}_i$ had a direction sampled uniformly at random and a speed uniformly randomly sampled from the range [0.5, 1] m/s.

\subsubsection{Robot Arm MT-TSP}\label{sec:ramt_tsp_instance_gen}
For the Robot Arm MT-TSP, we generated instances for a 7 DOF arm with the same geometry and joint angle limits as the KUKA iiwa, and we set $\mathcal{Q}_{\textnormal{a},0} = \Qagt$. All joints had the same max speed, denoted as $v_\text{max}$, in all instances, with a default value of 5 rad/s. All targets had time windows of length 3 s. Each segment of $\tilde{\tau}_i$ necessarily had a position component and orientation component. Each position component had a direction sampled uniformly at random and a speed uniformly random sampled from the range [0.5, 1] m/s. Each orientation component had a world-frame angular velocity whose $x$, $y$, and $z$ components were each sampled from the standard normal distribution.

\subsection{Tuning the Planners}\label{sec:tuning}
We tuned three values for each IRG-based planner for each problem variant: $n_\text{rand, init}$, used in Alg. \ref{alg:InitialTourGen}, $n_\text{rand}$, used in Alg. \ref{alg:IRG_PGLNS} and Alg. \ref{alg:PCG}, and $n_\text{term}$, used in Alg. \ref{alg:PGLNS}. GLNS and PGLNS have other tunable parameters, which we set equal to their fast mode values \cite{smith2017glns}.

\subsubsection{Tuning $n_\textnormal{rand, init}$}\label{subsubsec:tuning_nrand_init}
$n_\text{rand, init}$ is common to all planners within a problem variant, since they all use the same initial tour generation routine. We determined $n_\text{rand,init}$ by first generating a set of 30 \text{tuning instances} using default parameters, then we changed the target radius to zero in the Close-Enough MT-TSP, $v_\text{min} = 5$ in the Variable-Speed Dubins MT-TSP, and $v_\text{max} = 4.1$ for the Robot Arm MT-TSP, to ensure that $n_\text{rand,init}$ was tuned for the most constrained instances. We used a different set of random seeds for generating the tuning instances than the instances in later sections, so that none of the tuning instances were used in subsequent experiments.

For each variant, we ran \GenerateInitialTour on the tuning instances, setting $n_\text{rand,init} = 2$, setting the time limit to 1 hour so that a solution was found in all instances. After finding a solution for all 30 instances, we set $n_\text{rand,init}$ equal to the largest final value of $|\mathcal{S}_1|$ over all instances for that variant. The tuned $n_\text{rand, init}$ values for the Close-Enough MT-TSP, Variable-Speed Dubins MT-TSP, and Robot Arm MT-TSP were 8, 44, and 38, respectively.

\subsubsection{Tuning $n_\textnormal{rand}$ and $n_\textnormal{term}$}
To tune $n_\text{rand}$ and $n_\text{term}$, we generated another set of 10 tuning instances per problem variant using default parameters. Using the same strategy as \cite{smith2017glns}, we set $n_\text{term} = \alpha_\text{term} n_\text{cluster}$, where $n_\text{cluster}$ is the number of clusters in the GTSP. In all GTSPs we solved, $n_\text{cluster} = \ntar + 1$, since we have a cluster per target and a singleton cluster $\{(\qagtzero, 0)\}$. For each planner and problem variant, we ran the planner until the time budget on each tuning instance, for each combination of $n_\text{rand} \in \{1, 2, 4, 8, 16, 32, 64\}$ and $\alpha_\text{term} \in \{1, 2, 4, 8, 16, 32, 64\}$. For each instance and parameter setting, we computed the area under the curve (AUC) of the solution cost vs. planning time curve, as shown in Fig. \ref{fig:auc}. We chose the combination of $n_\text{rand}$ and $n_\text{term}$ that minimized the median AUC over the 10 tuning instances.
Table \ref{table:tuned_param_vals} shows the tuned $n_\text{rand}$ and $\alpha_\text{term}$ values.


For all planners, the tuned value of $\alpha_\text{term}$ is smaller than the original value of 60 in GLNS fast mode, except for IRG-GLNS in the Variable-Speed Dubins MT-TSP. The smaller $\alpha_\text{term}$ leads planners to resample more frequently and spend less time solving any particular GTSP than if they were using the original $\alpha_\text{term}$ from GLNS’s fast mode.

\begin{wrapfigure}[13]{r}{0.25\textwidth}
    \centering
    \vspace{-0.2cm}
    \includegraphics[width=0.25\textwidth]{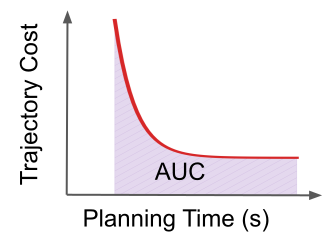}
    \vspace{-0.7cm}
    \caption{Performance metric: area under the curve (AUC). Small AUC is desirable, as it indicates fast convergence to a low-cost solution.}
    \label{fig:auc}
\end{wrapfigure}

\renewcommand{\arraystretch}{1.8}
\begin{table}
\vspace{-0.2cm}
\caption{Tuned parameters, shown as $(n_\text{rand}, \alpha_\textnormal{term})$, for IRG methods}\label{table:tuned_param_vals}
\vspace{-0.2cm}
\centering
\begin{tabular}{|c|c|c|c|}
\hline
 & \begin{tabular}{@{}c@{}}Close-Enough\vspace{-0.2cm}\\MT-TSP\end{tabular} & \begin{tabular}{@{}c@{}}Variable-Speed\vspace{-0.2cm}\\Dubins MT-TSP\end{tabular} & \begin{tabular}{@{}c@{}}Robot Arm\vspace{-0.2cm}\\MT-TSP\end{tabular}\\
\hline
\begin{tabular}{@{}c@{}}IRG-GLNS (ablation)\end{tabular} & (32, 4) & (32, 64) & (32, 2)\\
\hline
\begin{tabular}{@{}c@{}}PDG (ablation)\end{tabular} & (8, 1) & (2, 4) & (32, 4)\\
\hline
\begin{tabular}{@{}c@{}}IRG-PGLNS\end{tabular} & (16, 4) & (32, 32) & (32, 32)\\
\hline
\begin{tabular}{@{}c@{}}PCG\end{tabular} & (16, 4) & (4, 8) & (16, 4)\\
\hline
\end{tabular}
\vspace{-0.5cm}
\end{table}

\subsubsection{Tuning the Memetic Algorithm Baseline}\label{sec:memetic_tuning}
As described in Appendix \ref{appendix:memetic_modifications}, the memetic algorithm baseline, based on \cite{ding2022memetic}, performs several iterations, where each iteration performs operations on a population of solutions to generate a new population. We use the algorithm parameters from \cite{ding2022memetic}, except for the population size, since we found that with the population size from \cite{ding2022memetic}, in instances with 200 targets, the algorithm would spend most of its time constructing its initial population and little time improving the solutions in the population. \cite{ding2022memetic} did not have this issue because they only tested instances with up to 40 targets, and they did not consider time windows. Time windows make generating an initial population more computationally expensive, as discussed in Appendix \ref{appendix:memetic_modifications}. Therefore, we tune the population size as follows.

Using the strategy from \cite{ding2022memetic}, we set the population size equal to $\alpha_\text{pop}n_\text{tar}$ and tuned $\alpha_\text{pop}$. For the Close-Enough MT-TSP, we ran the baseline until the time budget on each tuning instance with $\alpha_\text{pop} \in \{0.01, 0.1, 1, 10\}$, then chose the $\alpha_\text{pop}$ that minimized the median AUC. We did the same for the Variable-Speed Dubins MT-TSP, but with $\alpha_\text{pop} \in \{0.01, 0.02, 0.03, 0.04\}$. We chose the set of attempted $\alpha_\text{pop}$ values for each variant so that the largest attempted $\alpha_\text{pop}$ resulted in the entire time budget being spent initializing the population. The tuned $\alpha_\text{pop}$ was $0.1$ for the Close-Enough MT-TSP and $0.04$ for the Variable-Speed Dubins MT-TSP.

\subsection{Varying the Number of Targets}\label{sec:vary_num_targets}
In this experiment, we varied the number of targets from 50 to 200 for all problem variants, with other instance parameters at their default values.
The AUC values for all numbers of targets are given in Table \ref{table:vary_num_targets_auc}, and the cost vs. time curves for 200 targets are shown in Fig. \ref{fig:cost_vs_time_200_targets}. 
All of the IRG-based planners outperform the memetic algorithm in min, median, and max AUC. For the Close-Enough MT-TSP, PCG has the smallest median AUC for all numbers of targets (Table \ref{table:vary_num_targets_auc} (a)). In the Variable-Speed Dubins MT-TSP, IRG-PGLNS has the smallest median AUC for 100 or more targets (Table \ref{table:vary_num_targets_auc} (b)). Finally, for the Robot Arm MT-TSP, IRG-PGLNS has the smallest median AUC for all numbers of targets (Table \ref{table:vary_num_targets_auc} (c)).

The memetic algorithm performs poorly largely due to the greedy optimization it uses to compute an agent trajectory for any fixed sequence of targets, i.e. the ``transformation" procedure described in Appendix \ref{appendix:memetic_modifications}. We found that this procedure often fails to generate feasible trajectories in the presence of time windows. The reason PCG noticeably outperforms IRG-PGLNS in the Close-Enough MT-TSP, but not for the other variants, is that parallelizing operations on a single sample point graph (sampling, computing edge costs, and solving a GTSP), as IRG-PGLNS does, yields smaller reductions in median iteration time for the Close-Enough MT-TSP than in the other variants. Section \ref{sec:vary_num_cores} discusses this further.

\begin{figure}
    \centering
    \includegraphics[width=0.5\textwidth]{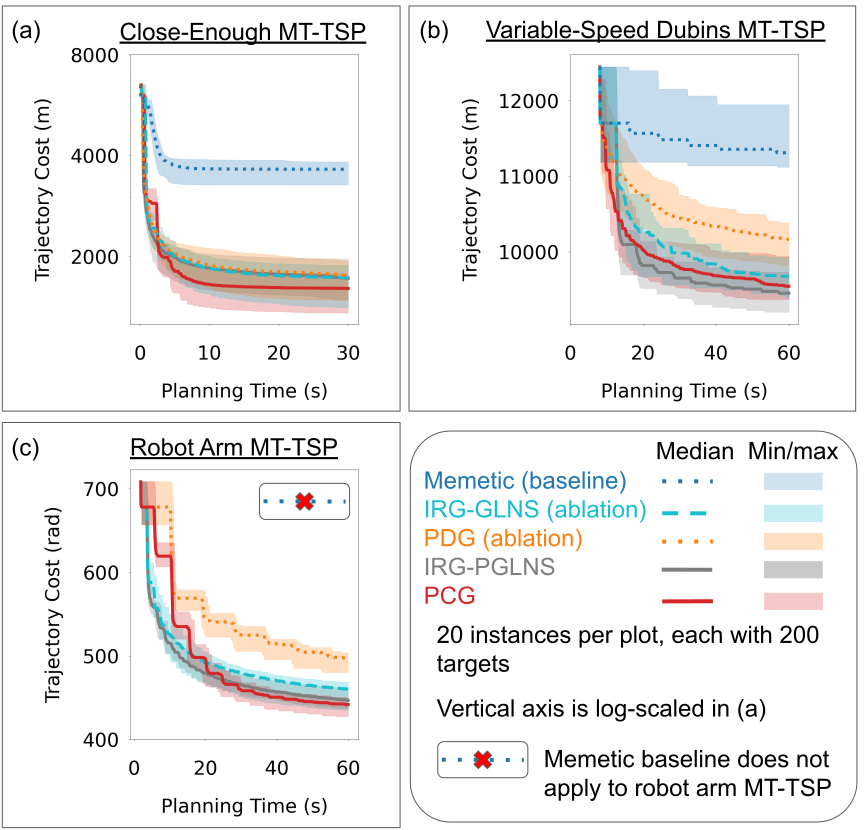}
    \vspace{-0.7cm}
    \caption{Cost vs. time for three MT-TSP variants, in instances with 200 targets. IRG-PGLNS and PCG find notably lower-cost solutions than the baseline and ablations, except for the Close-Enough MT-TSP (a), where IRG-PGLNS outperforms the memetic baseline while performing similarly to the ablations.}
    \label{fig:cost_vs_time_200_targets}
    \vspace{-0.7cm}
\end{figure}

\renewcommand{\arraystretch}{1.8}
\setlength{\tabcolsep}{3pt}
\begin{table}
\vspace{-0.2cm}
\caption{AUC values for MT-TSP variants. Entries are shown as min$|$median$|$max. Best min, median, and max in each column are bolded. Statistics are taken over 20 instances for each $\ntar$.}\label{table:vary_num_targets_auc}
\vspace{-0.2cm}
\centering
\begin{tabular}{|c|c|c|c|c|}
\hline
\multicolumn{5}{|c|}{(a) AUC (m $\cdot$ s) for Close-Enough MT-TSP}\\
\hline
\begin{tabular}{@{}c@{}}\end{tabular} & \begin{tabular}{@{}c@{}} $n_\text{tar} = 50$\end{tabular} & \begin{tabular}{@{}c@{}} $n_\text{tar} = 100$\end{tabular} & \begin{tabular}{@{}c@{}} $n_\text{tar} = 150$\end{tabular} & \begin{tabular}{@{}c@{}} $n_\text{tar} = 200$\end{tabular}\\
\hline
\begin{tabular}{@{}c@{}}Memetic\vspace{-0.1cm}\\(baseline)\end{tabular} & 18.8$|$26.2$|$30.7 & 47.8$|$54.4$|$60.8 & 74.0$|$82.4$|$95.5 & 103$|$112$|$119\\
\hline
\begin{tabular}{@{}c@{}}IRG-GLNS\vspace{-0.1cm}\\(ablation)\end{tabular} & 8.92$|$14.0$|$16.9 & 21.2$|$25.9$|$32.0 & 35.2$|$39.8$|$46.1 & 48.7$|$57.3$|$61.9\\
\hline
\begin{tabular}{@{}c@{}}PDG\vspace{-0.1cm}\\(ablation)\end{tabular} & 8.62$|$13.5$|$\textbf{15.9} & 20.9$|$25.6$|$32.3 & 36.1$|$40.2$|$46.0 & 49.6$|$58.5$|$64.2\\
\hline
\begin{tabular}{@{}c@{}}IRG-PGLNS\end{tabular} & 8.69$|$13.8$|$16.7 & 21.9$|$26.3$|$31.8 & 34.7$|$40.4$|$50.0 & 49.1$|$56.6$|$62.2\\
\hline
\begin{tabular}{@{}c@{}}PCG\end{tabular} & \textbf{8.62}$|$\textbf{12.5}$|$16.5 & \textbf{19.3}$|$\textbf{24.2}$|$\textbf{30.5} & \textbf{32.4}$|$\textbf{37.4}$|$\textbf{43.0} & \textbf{46.0}$|$\textbf{53.6}$|$\textbf{58.4}\\
\hline

\end{tabular}

\vspace{0.1cm}

\begin{tabular}{|c|c|c|c|c|}
\hline
\multicolumn{5}{|c|}{(b) AUC (m $\cdot$ s) for Variable-Speed Dubins MT-TSP}\\
\hline
\begin{tabular}{@{}c@{}}\end{tabular} & \begin{tabular}{@{}c@{}} $n_\text{tar} = 50$\end{tabular} & \begin{tabular}{@{}c@{}} $n_\text{tar} = 100$\end{tabular} & \begin{tabular}{@{}c@{}} $n_\text{tar} = 150$\end{tabular} & \begin{tabular}{@{}c@{}} $n_\text{tar} = 200$\end{tabular}\\
\hline
\begin{tabular}{@{}c@{}}Memetic\vspace{-0.1cm}\\(baseline)\end{tabular} & 143$|$169$|$185 & 283$|$329$|$350 & 437$|$468$|$500 & 579$|$591$|$624\\
\hline
\begin{tabular}{@{}c@{}}IRG-GLNS\vspace{-0.1cm}\\(ablation)\end{tabular} & \textbf{133}$|$\textbf{137}$|$146 & 265$|$272$|$\textbf{280} & 392$|$403$|$414 & 503$|$524$|$531\\
\hline
\begin{tabular}{@{}c@{}}PDG\vspace{-0.1cm}\\(ablation)\end{tabular} & 138$|$145$|$151 & 279$|$287$|$300 & 408$|$425$|$438 & 525$|$545$|$563\\
\hline
\begin{tabular}{@{}c@{}}IRG-PGLNS\end{tabular} & 134$|$138$|$\textbf{145} & \textbf{264}$|$\textbf{271}$|$281 & \textbf{388}$|$\textbf{396}$|$\textbf{409} & \textbf{495}$|$\textbf{509}$|$\textbf{519}\\
\hline
\begin{tabular}{@{}c@{}}PCG\end{tabular} & 136$|$142$|$150 & 269$|$278$|$293 & 390$|$404$|$419 & 499$|$516$|$528\\
\hline
\end{tabular}

\vspace{0.1cm}

\begin{tabular}{|c|c|c|c|c|}
\hline
\multicolumn{5}{|c|}{(c) AUC (rad $\cdot$ s) for Robot Arm MT-TSP}\\
\hline
\begin{tabular}{@{}c@{}}\end{tabular} & \begin{tabular}{@{}c@{}} $n_\text{tar} = 50$\end{tabular} & \begin{tabular}{@{}c@{}} $n_\text{tar} = 100$\end{tabular} & \begin{tabular}{@{}c@{}} $n_\text{tar} = 150$\end{tabular} & \begin{tabular}{@{}c@{}} $n_\text{tar} = 200$\end{tabular}\\
\hline
\begin{tabular}{@{}c@{}}IRG-GLNS\vspace{-0.1cm}\\(ablation)\end{tabular} & 6.44$|$6.79$|$7.39 & 13.2$|$13.9$|$14.4 & 20.4$|$21.0$|$21.4 & \textbf{27.2}$|$28.5$|$29.2\\
\hline
\begin{tabular}{@{}c@{}}PDG\vspace{-0.1cm}\\(ablation)\end{tabular} & 6.54$|$7.05$|$7.46 & 14.2$|$14.8$|$15.4 & 22.2$|$23.1$|$23.4 & 30.9$|$31.7$|$32.1\\
\hline
\begin{tabular}{@{}c@{}}IRG-PGLNS\end{tabular} & \textbf{6.22}$|$\textbf{6.64}$|$\textbf{7.07} & \textbf{12.9}$|$\textbf{13.5}$|$\textbf{13.9} & \textbf{19.9}$|$\textbf{20.7}$|$\textbf{21.1} & 27.4$|$\textbf{27.9}$|$\textbf{28.4}\\
\hline
\begin{tabular}{@{}c@{}}PCG\end{tabular} & 6.28$|$6.72$|$7.29 & 12.9$|$13.7$|$14.1 & 20.0$|$20.9$|$21.3 & 27.6$|$28.4$|$29.1\\
\hline
\end{tabular}
\vspace{-0.5cm}
\end{table}

\subsection{Varying the Number of Cores}\label{sec:vary_num_cores}
In this experiment, we varied the number of cores used by the IRG-based algorithms and the memetic baseline, denoted in plots as $n_\text{core}$. For IRG-GLNS, $n_\text{core}$ is the number of threads used for computing edge costs, as well as the number of processes used for sampling in the Robot Arm MT-TSP. For PDG and PCG, $n_\text{core}$ is $n_\text{proc}$. For IRG-PGLNS, $n_\text{core}$ is the number of threads used for edge cost computations, the number of processes used for sampling in the Robot Arm MT-TSP, and $n_\text{thread}$ in PGLNS. For the memetic algorithm, $n_\text{core}$ is the number of threads used in initial population construction and when generating a new population from a previous one (see Appendix \ref{appendix:memetic_modifications}).

For each algorithm, the first time that Alg. \ref{alg:InitialTourGen} is invoked, we use all the cores on our computer, so that the initial solution is found at approximately the same time in all experiments. This avoids counterintuitive trends where using fewer cores causes the initial solution computation to finish later and the AUC to decrease, which would give an impression that using fewer cores improves performance. In the IRG-based algorithms, Alg. \ref{alg:InitialTourGen} is only called once, whereas in the memetic baseline, Alg. \ref{alg:InitialTourGen} is called once for each member of the initial population; thus, the memetic baseline uses all cores to generate its first population member, then $n_\text{core}$ cores for subsequent members.

First, to elucidate why IRG-PGLNS underperformed PCG for the Close-Enough MT-TSP in Section \ref{sec:vary_num_targets}, we plot each component of IRG-PGLNS's computation time per iteration in Fig. \ref{fig:irg_pglns_timing_breakdown_vs_num_cores}. The total speedup when going from 1 to 10 cores is small for the Close-Enough MT-TSP when compared to the Variable-Speed Dubins and Robot Arm MT-TSPs, making it more useful to consider several sample point graphs at once, as PCG does, rather than applying parallelization to a single sample point graph, as IRG-PGLNS does. The reason that IRG-PGLNS has a small total speedup for the Close-Enough MT-TSP is that computing edge costs, which is embarrassingly parallel, is less expensive than solving the GTSP, which is not embarrassingly parallel, as shown in Fig. \ref{fig:irg_pglns_timing_breakdown_vs_num_cores} (a). In all the variants, parallelizing the GTSP solve tends to provide sublinear speedups, and thus in the Close-Enough MT-TSP, where the GTSP solve is the bottleneck, the overall speedup for IRG-PGLNS is small. Meanwhile, in the Variable-Speed Dubins MT-TSP, the bottleneck is computing edge costs (Fig. \ref{fig:irg_pglns_timing_breakdown_vs_num_cores} (b)), and in the Robot Arm MT-TSP, the bottleneck is sampling (Fig. \ref{fig:irg_pglns_timing_breakdown_vs_num_cores} (c)). Both of these bottlenecks are embarrassingly parallel, and parallelizing them leads to large reductions in computation time, leading IRG-PGLNS to perform well on the Variable-Speed Dubins and Robot Arm MT-TSPs.

\begin{figure}
    \centering
    \includegraphics[width=0.5\textwidth]{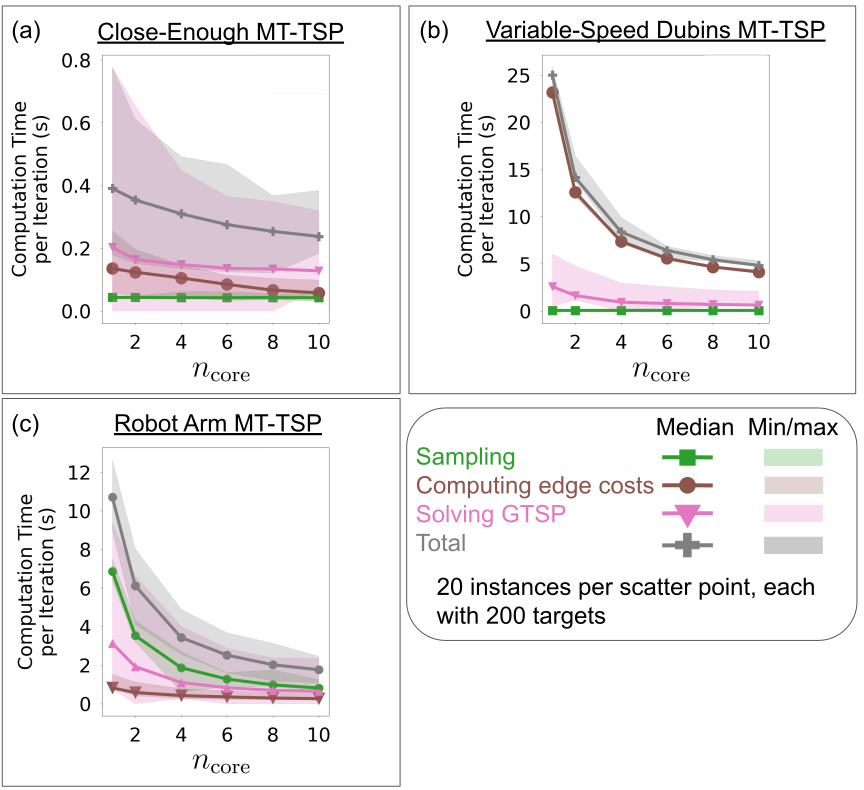}
    \vspace{-0.7cm}
    \caption{Components of runtime per iteration of IRG-PGLNS as we vary the number of cores it uses. Increasing the number of cores tends to provide a smaller reduction in runtime for the Close-Enough MT-TSP than in other variants, as discussed in Section \ref{sec:vary_num_cores}.}
    \label{fig:irg_pglns_timing_breakdown_vs_num_cores}
    \vspace{-0.5cm}
\end{figure}

Next, we examine the AUC for all planners as we vary the number of cores in Fig. \ref{fig:vary_num_cores_auc}. For all planners and problem variants, except for PDG, AUC tends to decrease with the number of cores used. PDG's displays little trend with respect to the number of cores used, highlighting the importance of communication between processes in PCG.

\begin{figure}
    \centering
    \includegraphics[width=0.5\textwidth]{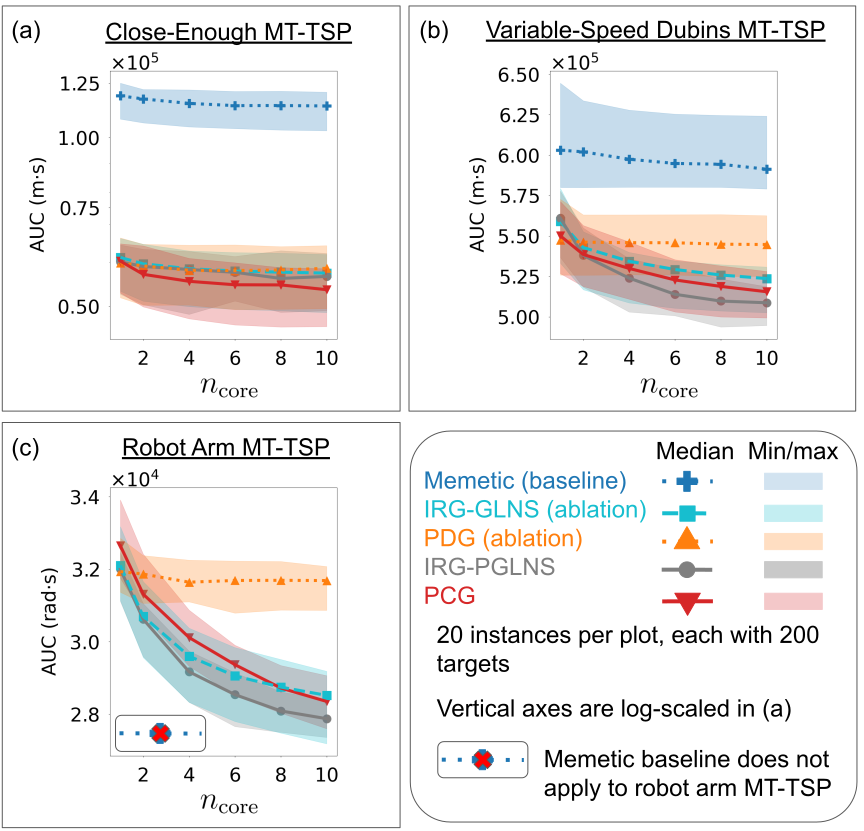}
    \vspace{-0.7cm}
    \caption{Varying the number of cores used by each algorithm. Methods' AUC values differ when they all use 1 core because they use different values of $n_\text{rand}$ and $\alpha_\text{term}$. Each planner's AUC tends to decrease as we increase the number of cores, with PDG and the memetic algorithm showing smaller decreases.}
    \label{fig:vary_num_cores_auc}
    \vspace{-0.5cm}
\end{figure}

\subsection{Close-Enough MT-TSP: Varying the Target Radius}\label{sec:cemt_tsp_vary_target_radius}
In this experiment, we varied the target radius in the Close-Enough MT-TSP from 0 to 12 m, leaving other instance parameters at their default values. The results are in Fig. \ref{fig:problem_specific_auc} (a). As the target radius becomes larger, the AUCs become smaller for all planners, since the instances become less constrained and have smaller optimal costs. For all target radii, PCG has the smallest min, median, and max AUC.

\subsection{Variable-Speed Dubins MT-TSP: Varying the Min Speed}\label{sec:vsdmt_tsp_vary_vmin}
In this experiment, we varied $v_\text{min}$ in the Variable-Speed Dubins MT-TSP from 2 m/s to 5 m/s, with other instance parameters at their default values. When $v_\text{min} = 5$, $v_\text{min}$ equals $v_\text{max}$ and we have a fixed-speed Dubins MT-TSP. The results are in Fig. \ref{fig:problem_specific_auc} (b). As $v_\text{min}$ increases, the median AUC increases for all planners, since the instances become more constrained and have higher optimal costs.

Planners' AUC values tend to become more similar as we increase $v_\text{min}$. This is expected, since the gap in minimum and maximum trajectory cost decreases as $v_\text{min}$ approaches $v_\text{max}$. That is, if the latest starting time window begins at $\underline{t}$, and the latest ending time window begins at $\overline{t}$, the minimum cost for an instance is $\underline{c} = v_\text{min}\underline{t}$, and the maximum cost is $\overline{c} = v_\text{max}\overline{t}$. As we increase $v_\text{min}$, $\underline{c}$ becomes closer to $\overline{c}$.

\subsection{Robot Arm MT-TSP: Varying the Max Joint Speed}\label{sec:ramt_tsp_vary_vmax}
In this experiment, we varied the arm's max joint speed in the Robot Arm MT-TSP from 4.1 rad/s to 5.0 rad/s, leaving other instance parameters at their default values. The results are in Fig. \ref{fig:problem_specific_auc} (c). As $v_\text{max}$ increases, the median AUC tends to decrease for all planners, since the instances become less constrained and have smaller optimal costs. IRG-PGLNS has the smallest median and max AUC for all values of $v_\text{max}$.

\begin{figure}
    \centering
    \includegraphics[width=0.5\textwidth]{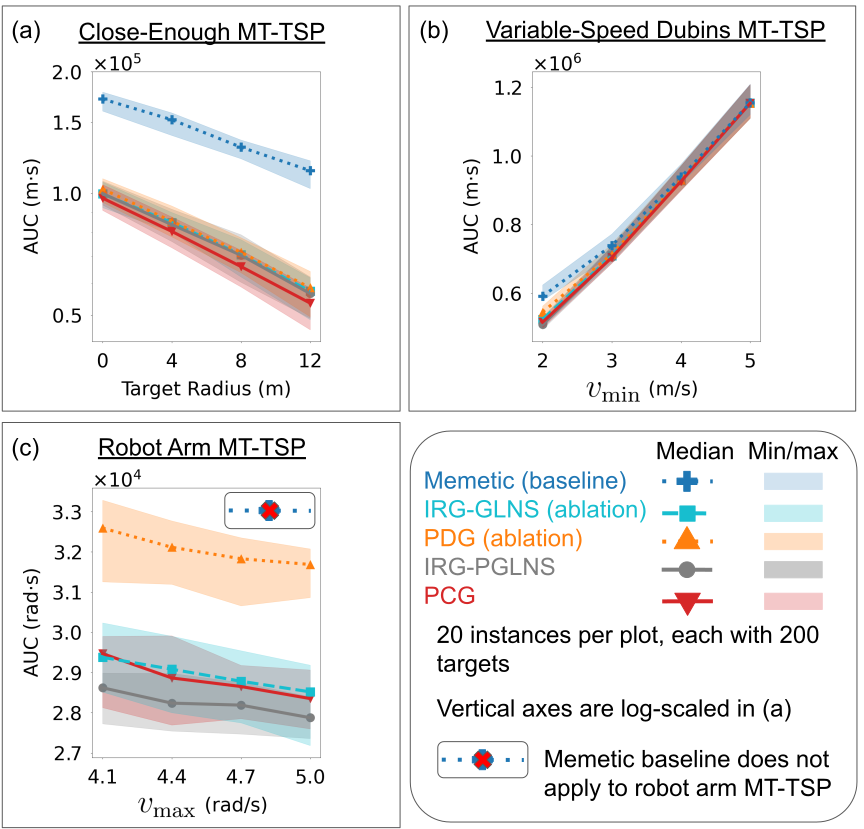}
    \vspace{-0.7cm}
    \caption{(a) Varying target radius in Close-Enough MT-TSP. PCG has smaller min, median, and max AUC than all other planners for all radii. (b) Varying $v_\text{min}$ in Variable-Speed Dubins MT-TSP. IRG-PGLNS and PCG perform similarly to each other and have smaller min, median, and max AUC than baselines and ablations for $v_\text{min} \leq 4$. (c) Varying $v_\text{max}$ in Robot Arm MT-TSP. IRG-PGLNS has smaller median and max AUC than all other planners for all $v_\text{max}$.}
    \label{fig:problem_specific_auc}
    \vspace{-0.5cm}
\end{figure}

\subsection{Linear MT-TSP: Evaluating Asymptotic Optimality}\label{sec:mt_tsp_eval_optimality}
In this section, we evaluate how quickly IRG-PGLNS and PCG converge to the optimal cost for the Linear MT-TSP, as defined in the beginning of Section \ref{sec:numerical_results}. In particular, we generated instances with 10 to 40 targets using the steps described in Section \ref{sec:generating_instances}, but we replaced the two-segment piecewise-linear trajectories with single-segment trajectories, and we set $\tau_i = \tilde{\tau}_i$ rather than fitting a B-spline to $\tilde{\tau}_i$. The instances used the same parameters as the instances for the Close-Enough MT-TSP, except we made the time window length 54 s rather than 108 s so that we could find the optimum using the optimal Mixed Integer Conic Program (MICP) \cite{philip2024mixedinteger} for up to 20 targets within a 1 minute time budget. We then ran the optimal MICP, IRG-PGLNS, and PCG on each instance for 1 minute. We also ran a planner we call IRG-IP, where IP stands for integer program; IRG-IP is identical to IRG-PGLNS, but it finds optimal GTSP solutions by solving the integer program from \cite{laporte1987generalized} using Gurobi. We show IRG-IP to see whether using an optimal GTSP solver, rather than PGLNS, practically accelerates convergence to the optimum.

For IRG-PGLNS and PCG, we modified PGLNS and GLNS by allowing the solvers to remove up to $n_\text{cluster} - 1$ nodes from a tour in each iteration, as opposed to the maximum of $0.1n_\text{cluster}$ removals used in GLNS fast mode. We did so because we found that for small numbers of clusters, only allowing $0.1n_\text{cluster}$ removals led to convergence to notably suboptimal solutions. We did not make this change for other experiments because for 100 or more targets, allowing $n_\text{cluster} - 1$ removals did not provide significant benefit. 

The results are in Fig. \ref{fig:sampling_vs_micp_gcs}. The optimal MICP certified that it found the optimal solution in all instances with 10 to 20 targets, but not in any instance with 30 to 40 targets. IRG-PGLNS and PCG both converged to solutions of similar quality to the optimal MICP. Also, both IRG-PGLNS and PCG exhibit faster convergence than IRG-IP. This demonstrates that the computational complexity of solving GTSPs with integer programming in the IRG framework outweighs the benefits of finding optimal GTSP solutions.

\begin{figure}
    \centering
    \includegraphics[width=0.5\textwidth]{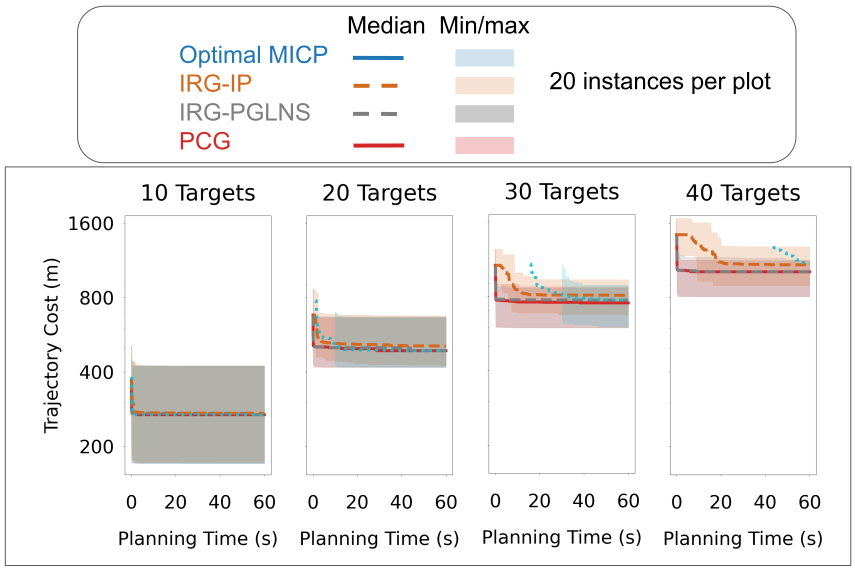}
    \vspace{-0.7cm}
    \caption{Testing IRG variations against optimal MICP solver \cite{philip2024mixedinteger} on ``linear" MT-TSP instances, as defined at the beginning of Section \ref{sec:numerical_results}. For instances with 10 and 20 targets, the optimal MICP found the optimal solution within the 1 min time limit, but on the instances with 30 and 40 targets, it did not certify that it found the optimum in any instance. IRG-PGLNS and PCG converge closely to the optimum for 10 and 20 targets and scale better with the number of targets than the optimal MICP and IRG-IP.}
    \label{fig:sampling_vs_micp_gcs}
    \vspace{-0.5cm}
\end{figure}

\def\EFAT{{\textnormal{EFAT}}}
Finally, we show that Assumption \ref{assumption:N_existence}, required for asymptotic optimality, holds for all Linear MT-TSP instances where the MICP found the optimum. In the Linear MT-TSP, the cost of a solution only depends on the interception times at the targets, so the set $\mathcal{H}$ from Section \ref{sec:theoretical_analysis} can be set arbitrarily, e.g. $\mathcal{H} = \{0\}$. We can then let $\mathcal{L}_\mathcal{H} = \prod\limits_{i = 1}^\ntar [\underline{t}_i, \overline{t}_i]$, ignoring $\mathcal{H}$. 
Let $S$ be the optimal target sequence, computed by the MICP, where $S[k]$ is the $k$th target visited. For convenience, we define a fictitious target 0 with $\tau_0(t) = q_{\text{a}, 0}$ for all $t$, and we let $S[0] = 0$.
The optimal interception times for $S$ can be found by solving the second-order cone program (SOCP) given in \eqref{optprob:fixed_target_seq_socp}, where the interior of the feasible set is a set $\mathcal{N}$ satisfying Assumption \ref{assumption:N_existence}.
\SetKwFunction{EFAT}{EFAT}
\begin{mini!}
{\{t_{i}\}_{i = 1}^{\ntar}}{\sum\limits_{k = 1}^{\ntar}\|\tau_{S[k]}(t_{S[k]}) - \tau_{S[k - 1]}(t_{S[k - 1]})\|\label{eqn:fixed_target_seq_objective}}{\label{optprob:fixed_target_seq_socp}}{}
\addConstraint{\underline{t}_i \leq t_i \; \forall i \in \mathcal{I}\label{eqn:fixed_target_seq_tw_lower}}{}
\addConstraint{t_{i} \leq \overline{t}_{i} \; \forall i \in \mathcal{I}\label{eqn:fixed_target_seq_tw_upper}}{}
\addConstraint{\|(\tau_{S[k]}(t_{S[k]}) - \tau_{S[k - 1]}(t_{S[k - 1]}))\| \leq \nonumber}{}
\addConstraint{}{\quad v_\text{max}(t_{S[k]} - t_{S[k - 1]}) \; \forall k \in \{1, 2, \dots, \ntar\}\label{eqn:fixed_target_seq_speed_limit}}
\end{mini!}
The decision variables of \eqref{optprob:fixed_target_seq_socp} are the interception times $t_i$ for each target $i$. \eqref{eqn:fixed_target_seq_objective} minimizes distance traveled, \eqref{eqn:fixed_target_seq_tw_lower}-\eqref{eqn:fixed_target_seq_tw_upper} enforce time window constraints, and \eqref{eqn:fixed_target_seq_speed_limit} enforces the speed limit. Let $\mathcal{P}$ be \eqref{optprob:fixed_target_seq_socp}'s feasible set. Let $\mathcal{N}$ be the interior of $\mathcal{P}$. To satisfy Assumption \ref{assumption:N_existence}, we need $\mathcal{N} \neq \emptyset$. In general, the optimal solution will not be in $\mathcal{N}$, because optimal solutions often lie on the boundary of a feasible set rather than the interior.
Thus, we used the following procedure to find an element of $\mathcal{N}$. We initialized a value $\delta = 10$. Then we iterated from $k = 1$ to $\ntar$, and for each $k$, set $t_{S[k]}$ to $\EFAT(S[k - 1], t_{S[k - 1]}, S[k]) + \delta$. $\EFAT(S[k - 1], t_{S[k - 1]}, S[k])$ denotes the earliest time at which the agent can intercept $S[k]$ after departing $S[k - 1]$ at time $t_{S[k - 1]}$, and can be computed using the method from \cite{philip2025CStar}. If any iteration of this loop failed due to $t_{S[k]} > \overline{t}_{S[k]}$, we divided $\delta$ by 2 and restarted.

Using the computed $t_i$ values for each $i \in \mathcal{I}$, we computed the distance of the LHS from the RHS for each constraint in \eqref{optprob:fixed_target_seq_socp}. All distances were larger than 0.039, i.e. far enough above machine precision to consider them nonzero. This means in each instance, $(t_1, t_2, \dots, t_\ntar) \in \mathcal{N}$, so $\mathcal{N} \neq \emptyset$.

Since $\mathcal{P}$ is convex with nonempty interior $\mathcal{N}$, $\overline{\mathcal{N}} = \mathcal{P}$ (Theorem 2.13 in \cite{Mordukhovich2022Convex}), so $\overline{\mathcal{N}}$ contains the optimal $l_\mathcal{H} \in \mathcal{L}_\mathcal{H}$ for $S$. The optimality of $S$ then implies Assumption \ref{assumption:N_existence} (i). All lists in $\overline{\mathcal{P}}$ correspond to the same target sequence $S$, satisfying Assumption \ref{assumption:N_existence} (ii). $\mathcal{N} \subseteq \mathcal{P}$ implies Assumption \ref{assumption:N_existence} (iii). Thus, Assumption \ref{assumption:N_existence} is satisfied in practice for the Linear MT-TSP. We cannot validate Assumption \ref{assumption:N_existence} in this way for the other MT-TSP variants because we do not know the optimal target sequence.
However, Assumption \ref{assumption:N_existence}'s validity for the Linear MT-TSP suggests that it is reasonable for the other variants.

\subsection{Evaluating Initial Tour Generation}\label{sec:mt_tsp_eval_init_tour_gen}
In this section, we evaluate the effectiveness of the DAG-DFS in Alg. \ref{alg:InitialDFS} for solving the GTSP in Alg. \ref{alg:InitialTourGen}, when generating the initial tour in IRG. All methods we compare against run Alg. \ref{alg:InitialTourGen}, but use a different method of solving the GTSP. Our first baseline, called ``Integer Program," solves the GTSP using an integer program \cite{laporte1987generalized}. Our second baseline, called ``GLNS," solves the GTSP using GLNS \cite{smith2017glns}. Our third baseline, called ``PGLNS," solves the GTSP using PGLNS. We also have an ablation called ``DAG-DFS-no-prune," which uses DAG-DFS (Alg. \ref{alg:InitialDFS}), but without the BEFORE sets for pruning. Finally, DAG-DFS uses Alg. \ref{alg:InitialDFS} as is.
All methods stop when they find a feasible solution.

In Section \ref{subsec:tour_via_gtsp}, we set the cost of infeasible edges for GLNS and PGLNS using the cost of the best tour generated so far. Since we do not have a best tour for this experiment--the experiment's purpose is to test methods of generating an initial tour--we set the infeasible edge costs equal to $(\ntar + 1)\overline{c}$, where $\overline{c}$ is the largest feasible edge cost, scaled and rounded as in Section \ref{subsec:tour_via_gtsp}. Since we do not have a seed tour to pass GLNS or PGLNS, we use a ``random insertion tour," as described in \cite{smith2017glns}. If GLNS or PGLNS reaches its termination criteria without finding a feasible tour, and there is time left, we restart the algorithm with a new random insertion tour.

We used the tuned $n_\text{rand,init}$ values for each MT-TSP variant stated in Section \ref{sec:tuning}, and Alg. \ref{alg:InitialTourGen} never needed to sample additional points after sampling the initial set of points. For GLNS and PGLNS, we used the parameters from GLNS's fast mode. We varied the number of targets and set all other instance parameters to their default values. The computation time results are in Fig. \ref{fig:compare_init_tour_methods}. Our DAG-DFS method tends to find initial solutions more quickly than the other methods, particularly for large numbers of targets. Comparisons against DAG-DFS-no-prune show that pruning based on the BEFORE sets is key to achieving low runtimes.

\begin{figure}
    \centering
    \includegraphics[width=0.5\textwidth]{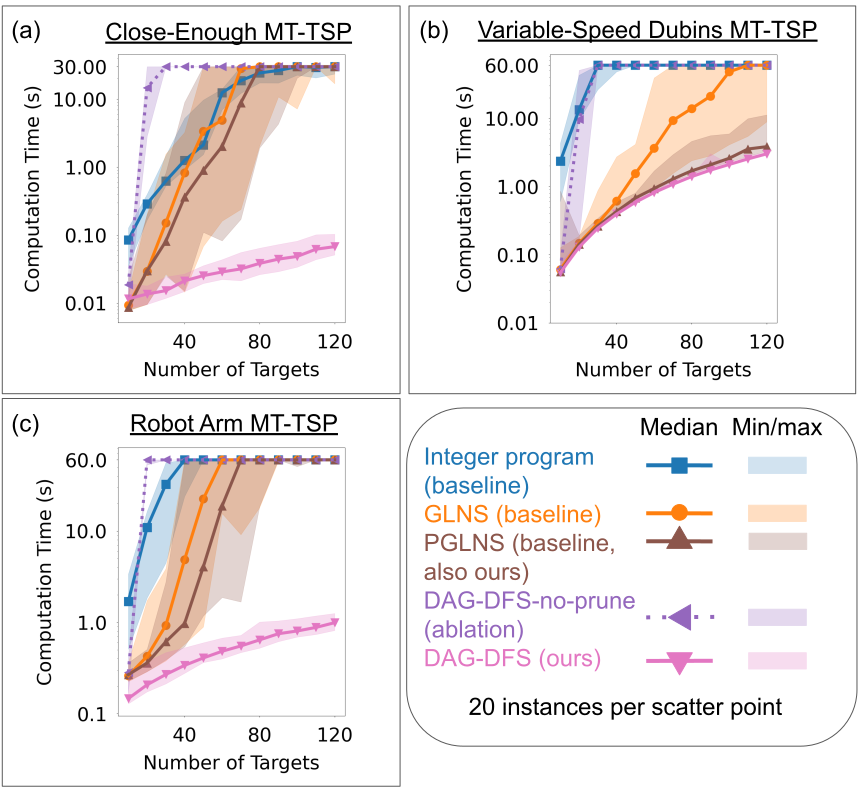}
    \vspace{-0.7cm}
    \caption{Comparison of computation time for methods for solving the GTSP on the first set of sample points. Our DAG-DFS method, used in all the IRG variants, scales better with the number of targets than the other methods.}
    \label{fig:compare_init_tour_methods}
    \vspace{-0.3cm}
\end{figure}

Next, we ran tour improvement via IRG-PGLNS for the remainder of the time budget for each instance, with each initial tour generation method. Table \ref{table:initial_tour_cost_comparison} compares the tour cost after improving tours from the different methods, for instances where the baselines found tours (DAG-DFS found tours in all instances). We do not compare cost against the ablation because DAG-DFS obtained the same initial tour as the ablation wherever the ablation found a tour. Improving DAG-DFS tours gave similar final cost as improving the baseline tours in the median case.
We also computed the statistics in Table \ref{table:initial_tour_cost_comparison} for each specific MT-TSP variant/baseline pair, and the median percent difference was less than 1\% in every case; we omit the numbers for brevity.
Thus, DAG-DFS gives similar solution quality to the baselines on small instances and scales to larger instances.
\renewcommand{\arraystretch}{1.8}
\begin{table}
\caption{Comparison of cost after tour improvement using the DAG-DFS initial solution $(c_\text{DAG-DFS})$ against cost after tour improvement using baseline initial solution $(c_\text{baseline})$.}\label{table:initial_tour_cost_comparison}
\vspace{-0.2cm}
\centering
\begin{tabular}{|c|c|c|c|}
\hline
 & Min & Median & Max \\
\hline
$\frac{c_\text{DAG-DFS} - c_\text{baseline}}{c_\text{baseline}}*100\%$ & -67.1\% & -0.00413\% & 30.2\%\\
\hline
\end{tabular}
\vspace{-0.7cm}
\end{table}

To examine DAG-DFS's scalability, we ran another experiment where we started at 400 targets, then continued to double the number of targets, running DAG-DFS on 20 instances for each number of targets. We recorded the largest number of targets for each MT-TSP variant where DAG-DFS found a solution within the time limit on all 20 instances. This number was 800 for the Close-Enough MT-TSP, 400 for the Variable-Speed Dubins MT-TSP, and 800 for the Robot Arm MT-TSP.

\section{Physical Robot Experiments}\label{sec:physical_robot_experiments}
In this section, we present results from our deployment of the IRG framework on a Turtlebot 4 Standard. The multimedia attachment shows a video of the experimental setup. We considered the Close-Enough MT-TSP. Our planners for the Close-Enough MT-TSP parameterize a trajectory as a tour $((\qagtzero, 0), (q_{\text{a}, 1}, t_1), (q_{\text{a}, 2}, t_2), \dots, (q_{\text{a}, \ntar}, t_\ntar))$, assuming the robot takes a straight-line path between consecutive pairs of points in the tour. Our differential-drive robot tracks such a trajectory as follows. First, the robot initializes itself to bring its estimated position within 0.1 m of $\qagtzero$, then bring its estimated yaw within 0.1 rad of $\pi$. Then at $t = 0$, the robot rotates to face $q_{\text{a}, 1}$, i.e. it rotates so that its estimated heading vector is aligned with the straight-line path to $q_{\text{a}, 1}$, within 0.1 rad. The robot uses proportional control for this. The robot then runs pure pursuit to bring its estimated position within 0.1 m of $q_{\text{a}, 1}$. Next, the robot waits until $t_1$; while waiting, it rotates to face $q_{\text{a}, 2}$. Once time $t_1$ is reached, and the robot is facing $q_{\text{a}, 2}$, the robot moves toward $q_{\text{a}, 2}$, and so forth.

In Section \ref{sec:problem_setup}, we defined \TrajExists assuming that the agent could change its motion direction instantaneously. Since this is not true on our robot, we modified \TrajExists($\qagt, t, \qagt', t')$ to check if $\|\qagt - \qagt'\| \leq v_\text{max}(t' - t - t_\text{rot})$, where $t_\text{rot}$ is the maximum time we expect the robot needs to rotate at one interception point to face the next one.
Based on our characterization of the robot, we set $t_\text{rot} = 3.5$ s.

We generated 5 instances with 50 targets each, using the method in Section \ref{sec:generating_instances}. We set $\mathcal{Q}_{\text{a},0} = [-1.746375 \text{ m}, 1.746375 \text{ m}] \times [-0.9225 \text{ m}, 0.9225 \text{ m}]$ and $v_\text{max} = 0.46$ m/s (the Turtlebot's maximum speed). Each segment of $\tilde{\tau}_i$ had a direction sampled uniformly at random and a speed sampled uniformly at random from the interval [0.046, 0.092] m/s. Each time window had length 20 s.

Due to localization error and trajectory-tracking error, when the robot is commanded to move to point $(q_{\text{a},k}, t_k)$, it generally will not reach the $q_{\text{a}, k}$ exactly.
Thus, when planning, we must reduce each target's disc radius by the maximum distance that we expect the robot to be from $q_{\text{a},k}$ at $t_k$.
We set the original disc radius for each target to 0.51 m, then applied a 0.1 m reduction due to the controller tolerance described above, as well as a 0.13 m reduction equal to the maximum position estimation error we have observed with respect to ground truth for our robot. Thus, the planner used a disc radius of 0.28 m.

Note that our radius reduction did not account for execution latency. In particular, consider a robot that takes significant time to stop after being commanded to stop. Then we may command the robot to stop when it gets within 0.1 m of $q_{\text{a}, k}$, but the robot may be more than 0.1 m from $q_{\text{a}, k}$ after actually stopping. For such a system, we would need to further reduce the disc radius by the maximum distance we expect the robot to travel after being commanded to stop. In our case, however, even without such a reduction, the estimated trajectory of the robot intercepted all targets in every experiment.


We planned a trajectory for each instance using PCG, IRG-GLNS, PDG, and the Memetic baseline offline with a 1 s planning time budget, and then tracked each planned trajectory on the robot. We did not consider IRG-PGLNS, since it showed little advantage for the Close-Enough MT-TSP in Section \ref{sec:numerical_results}, and was more useful in the other variants. Table \ref{table:robot_experiment_results} shows the results. PCG outperformed the other planners in both planned and executed trajectory costs.

\renewcommand{\arraystretch}{1.8}
\setlength{\tabcolsep}{3pt}
\begin{table}
\vspace{-0.2cm}
\caption{Costs (distance traveled) of planned trajectories, and trajectories executed on robot, in meters. Entries are shown as min$|$median$|$max. Best min, median, and max in each column are bolded. Statistics are taken over 5 instances.}\label{table:robot_experiment_results}
\vspace{-0.2cm}
\centering
\begin{tabular}{|c|c|c|}
\hline
\begin{tabular}{@{}c@{}}\end{tabular} & \begin{tabular}{@{}c@{}} Planned\end{tabular} & \begin{tabular}{@{}c@{}} Executed\end{tabular}\\
\hline
Memetic (baseline) & 25.4$|$29.9$|$32.8 & 27.1$|$32.5$|$34.7\\
\hline
IRG-GLNS (ablation) & 21.9$|$24.5$|$26.8 & 23.5$|$26.2$|$29.3\\
\hline
PDG (ablation) & 22.0$|$22.9$|$27.3 & 23.3$|$25.5$|$29.2\\
\hline
PCG & \textbf{20.4}$|$\textbf{21.5}$|$\textbf{26.4} & \textbf{21.6}$|$\textbf{23.6}$|$\textbf{28.6}\\
\hline
\end{tabular}
\vspace{-0.5cm}
\end{table}

\section{Conclusion}\label{sec:conclusion}
In this paper, we introduced the IRG framework for variants of the MT-TSP, as well as two parallel algorithms in this framework: IRG-PGLNS and PCG. We proved our algorithms' asymptotic optimality and demonstrated their advantages over a baseline and two ablations in numerical experiments and physical robot experiments. One direction for future work is to incorporate informed sampling methods, such as \cite{faigl2017solution}, into IRG, to accelerate convergence toward optimal solutions. Another direction, as we mention in Section \ref{sec:theoretical_analysis}, is to prove stronger finite-time performance guarantees for IRG.

\bibliographystyle{IEEEtran}
%
\bibliography{refs}

{\appendices
\section{Modification of Memetic Algorithm}\label{appendix:memetic_modifications}
In this section, we review the memetic algorithm from \cite{ding2022memetic} for the Dubins Close-Enough MT-TSP. We then explain how we extend \cite{ding2022memetic} to deal with time windows, parallelize \cite{ding2022memetic}, and specialize \cite{ding2022memetic} to the Close-Enough MT-TSP and the Variable-Speed Dubins MT-TSP, so that we can use it as a baseline.

\subsection{Review of Unmodified Memetic Algorithm}\label{subsec:memetic_review}
\cite{ding2022memetic} encodes a solution as a sequence $G = (g_1, g_2, \dots, g_{n_\text{tar}})$, where $g_k = (i_k, \theta_k, \Delta t_k)$. $i_k \in \mathcal{I}$ is a target index. $\theta_k \in [0, 2\pi]$ is an angle parameterizing a position along target $i_k$'s disc boundary, and $\Delta t_k$ is the agent's travel duration from target $i_{k - 1}$ to target $i_k$ (or for $k = 1$, the travel duration from $\qagtzero$ to $i_1$). 
Together, $\theta_k$ and $\Delta t_k$ indicate that the agent should intercept target $k$ at time $t_k = \sum\limits_{l=1}^{k}\Delta t_l$ and position $p_k = \tau_{i_k}(t_k) + r_{i_k}[\cos\theta_k, \sin\theta_k]^T$.
While \cite{ding2022memetic} considers a Dubins car, the solution encoding does not specify the agent's heading angle at each target and thus could map to several different agent trajectories. To decode a sequence $G = (g_1, g_2, \dots, g_{n_\text{tar}})$ into a single agent trajectory, \cite{ding2022memetic} begins with a zero-duration trajectory starting at $\qagtzero$ and time $0$. \cite{ding2022memetic} then iterates from $k = 1$ to $n_\text{tar}$, where each iteration appends a new trajectory segment to the existing trajectory, where the new segment ends at some configuration $q_{\text{a}, k}$. To generate this segment and $q_{\textnormal{a}, k}$, \cite{ding2022memetic} computes a path from $q_{\text{a}, k - 1}$ to $p_k$, with free terminal heading, curvature no larger than $\frac{1}{\rho}$ and length $v\Delta t_k$, where $v$ and $\rho$ are the fixed speed and fixed minimum turning radius considered by \cite{ding2022memetic}. The trajectory segment for iteration $k$ is obtained by moving the agent along the computed path with speed $v$. Letting the terminal heading of the path be $\phi_{\text{a}, k}$, we have $q_{\text{a}, k} = (p_k, \phi_{\text{a}, k})$.

\cite{ding2022memetic} begins by initializing a population of $n_\text{pop}$ solutions, $P = \{G_1, G_2, \dots G_{n_\text{pop}}\}$, then at each iteration, generates a new population by applying the following operations to each solution $G_m$ in the current population, in order: crossover, mutation, repair, and transformation. Crossover randomly selects another solution $G_n$ from $P$, then constructs a solution ${}^\text{crs}G_m$ using portions of $G_m$ and $G_n$. Mutation randomly perturbs ${}^\text{crs}G_m$ to generate a solution ${}^\text{mut}G_m$. Since crossover and mutation may produce an infeasible solution (i.e. where no curvature-bounded path with length $v\Delta t_k$ exists from $q_{\text{a}, k}$ to $p_k$ in the decoding procedure above), repair is applied to ${}^\text{mut}G_m$ to generate a feasible solution ${}^\text{rep} G_m$. 

To repair a solution ${}^\text{mut}G_m = (g_1, g_2, \dots, g_{n_\text{tar}})$, \cite{ding2022memetic} iterates from $k = 1$ to $n_\text{tar}$, where each iteration $k$ finds a $\Delta t_k$ such that a curvature-bounded path with length $v\Delta t_k$ exists from $q_{\text{a}, k - 1}$ to $p_k$. We refer to a $\Delta t_k$ satisfying these conditions as \emph{kinematically feasible}. \cite{ding2022memetic} finds $\Delta t_k$ by applying Newton's method to the root-finding problem $v\Delta t_k - d(\Delta t_k) = 0$, where $d(\Delta t_k)$ is the length of the shortest Dubins path with free terminal heading from $q_{\text{a}, k - 1}$ to $\tau_{i_k}(t_{k - 1} + \Delta t_k)$. \cite{ding2022memetic} terminates Newton's method when it finds a kinematically feasible $\Delta t_k$, which is not necessarily a root.

After generating a solution ${}^\text{rep} G_m$ via repair, \cite{ding2022memetic} performs a transformation step, which aims to modify the $\Delta t_k$ values in ${}^\text{rep} G_m$ to produce a lower-cost solution ${}^\text{trans} G_m$. The transformation iterates from $k = 1$ to $n_\text{tar}$ and attempts to minimize $\Delta t_k$ with other values in the encoding fixed. The minimization is formulated as the same root-finding problem as the repair step, but now \cite{ding2022memetic} iterates until $|v\Delta t_k - d(\Delta t_k)| < 0.01$. After generating ${}^\text{trans} G_m$. for each $G_m$, \cite{ding2022memetic} applies local search methods to improve the $\theta_k$ values of a subset of solutions, randomly selected from the top 50\%.

\subsection{Modification to Handle Time Windows}\label{sec:memetic_time_windows}
The first part of \cite{ding2022memetic} that we modify to handle time windows is the initial population construction. To construct an initial population, \cite{ding2022memetic} generates $n_\text{pop}$ solutions by randomly sampling different sequences of $g_k$ values. We found that in the presence of time windows, this method often generated infeasible solutions, so we instead generate the initial population by running Alg. \ref{alg:InitialTourGen} $n_\text{pop}$ times with different random seeds. 

We additionally modify the repair step to handle time windows. In particular, when using Newton's method for root-finding at iteration $k$, we clip $\Delta t_k$ iterates to the range $[\underline{t}_{i_k} - t_{k - 1}, \overline{t}_{i_k} - t_{k - 1}]$. It is possible that there is no kinematically feasible $\Delta t_k \in [\underline{t}_{i_k} - t_{k - 1}, \overline{t}_{i_k} - t_{k - 1}]$. To handle this complication, before running Newton's method, we check if $t_{k - 1} > \overline{t}_{i_k}$, and if so, we terminate the repair step. We discuss additional termination conditions for the repair that are specific to each problem variant in Appendix \ref{sec:memetic_ce} and \ref{sec:memetic_vs_dubins}. If we terminate the repair without obtaining a kinematically feasible $\Delta t_k \in [\underline{t}_{i_k} - t_{k - 1}, \overline{t}_{i_k} - t_{k - 1}]$, we skip the transformation step and set ${}^\text{trans}G_m$ equal to $G_m$. We also modify the transformation step to handle time windows in ways specific to each problem variant, discussed in Appendix \ref{sec:memetic_ce} and \ref{sec:memetic_vs_dubins}.

\subsection{Parallelization}
We parallelize the initial population construction of \cite{ding2022memetic} by parallelizing Alg. \ref{alg:InitialTourGen}, using the methods described in Section \ref{sec:numerical_results}. We also perform crossover, mutation, repair, and transformation on different solutions $G_k$ in different threads. Finally, we parallelize the local search as follows. In one of its local search methods, \cite{ding2022memetic} chooses some $k \in \mathcal{I}$ at random, samples different candidate values $\tilde{\theta}_k^1, \tilde{\theta}_k^2, \dots, \tilde{\theta}_k^{n_\text{sample}} \in [0, 2\pi]$, then sets $\theta_k$ to the sample giving the largest cost improvement. We evaluate the cost improvement for all the samples in parallel.

\subsection{Specialization to Close-Enough MT-TSP}\label{sec:memetic_ce}
Our specialization of \cite{ding2022memetic} to the Close-Enough MT-TSP relies on the fact that in all instances we generated, targets' speeds are no larger than the agent's maximum speed. We checked this condition for each instance as follows. First, we computed the control points of the velocity B-spline for each target. Next, we computed the norm of each control point, and we let $\bar{v}_i$ be the maximum of these norms for target $i$.
We verified that $\bar{v}_i \leq v_\text{max}$ for all targets, which ensures that targets move with speed no faster than $v_\text{max}$.

Before running Newton's method during iteration $k$ of the repair step, we check if the agent can meet target $i_k$ at the end of its time window, i.e. we check if $\|\tau_{i_k}(\overline{t}_{i_k}) - q_{\text{a}, k - 1}\| \leq v_\text{max}(\overline{t}_{i_k} - t_{k - 1})$. If not, $\bar{v}_{i_k} \leq v_\text{max}$ implies that the agent cannot meet target $i_k$ within its time window (see \cite{philip2025CStar}, Theorem 2), and we terminate the repair step. On the other hand, if $\|\tau_{i_k}(\overline{t}_{i_k}) - q_{\text{a}, k - 1}\| \leq v_\text{max}(\overline{t}_{i_k} - t_{k - 1})$, we run Newton's method, starting at the value of $\Delta t_k$ from ${}^\text{mut}G_m$, until we find a kinematically feasible $\Delta t_k \in [\underline{t}_{i_k} - t_{k - 1}, \overline{t}_{i_k} - t_{k - 1}]$.

Finally, we modified the transformation procedure for the Close-Enough MT-TSP. The transformation from \cite{ding2022memetic} attempts to minimize $\Delta t_k$, which only minimizes distance traveled for a fixed-speed agent. To handle a variable-speed agent, we use projected gradient descent to find $\Delta t_k$ that minimizes the straight-line distance from $q_{\text{a}, k - 1}$ to $\tau_{i_k}(t_{k - 1} + \Delta t_k)$.

\subsection{Specialization to Variable-Speed Dubins MT-TSP}\label{sec:memetic_vs_dubins}
To specialize \cite{ding2022memetic} to the Variable-Speed Dubins MT-TSP, we remove the $\theta_k$ values from the solution encoding, and thus we remove the local search on the $\theta_k$ values as well. Additionally, during iteration $k$ of the repair step, we perform the root-finding described in Appendix \ref{subsec:memetic_review} for each $v \in \mathcal{S}_\text{speed}$, rather than for a single speed. We run Newton's method until either (i) $|v\Delta t_k - d(\Delta t_k)| < 0.01$, (ii) two consecutive $\Delta t_k$ iterates are the same, or (iii) we reach a maximum number of iterations $n_\text{rep}$. To determine $n_\text{rep}$, we ran our modification of \cite{ding2022memetic} on 10 instances with 50 targets, initially with $n_\text{rep} = 1000$. Newton's method never needed more than 276 iterations to succeed, so we rounded to the nearest hundred and set $n_\text{rep} = 300$.

After performing a Newton solve for each $v$ at iteration $k$ of a repair, we have up to one $\Delta t_k$ for each $v$. If we have a $\Delta t_k$ for any $v$, we choose the smallest $\Delta t_k$. This procedure attempts to minimize $\Delta t_k$, rather than the method from \cite{ding2022memetic}, which only seeks a kinematically feasible $\Delta t_k$. While we initially attempted methods that stopped the $k$th repair iteration as soon as they found a kinematically feasible $\Delta t_k$, this resulted in fewer repairs succeeding and larger solution costs. This is in contrast to the Close-Enough MT-TSP, where we found that minimizing $\Delta t_k$ up to a tolerance resulted in higher cost than stopping at a kinematically feasible $\Delta t_k$.

We also modified the transformation as follows. During iteration $k$, while \cite{ding2022memetic} attempts to minimize $\Delta t_k$ for a fixed speed, we perform minimization for each $v \in \mathcal{S}_\text{speed}$, then choose the $v$ giving the minimum $v\Delta t_k$, aiming to minimize distance traveled. We solve the underlying root-finding problem with Newton's method and use the same clipping and termination conditions as in the repair step to handle time windows.
}

\end{document}